\documentclass{article}
\usepackage{graphicx}%
\usepackage{multirow}%
\usepackage{amsmath,amssymb,amsfonts}%
\usepackage{amsthm}%
\usepackage{mathrsfs}%
\usepackage[title]{appendix}%
\usepackage{xcolor}%
\usepackage[table]{xcolor}
\usepackage{textcomp}%
\usepackage{manyfoot}%
\usepackage{booktabs}%
\usepackage{listings}%
\usepackage[colorlinks=true,allcolors=black]{hyperref}
\usepackage{url}
\usepackage{thmtools}
\usepackage{thm-restate}
\usepackage{mathrsfs}%
\usepackage[title]{appendix}%
\usepackage{adjustbox}
\usepackage{makecell}
\usepackage{microtype}
\usepackage{subcaption}
\usepackage{mathtools}
\usepackage[capitalize,noabbrev]{cleveref}
\usepackage[preprint,nohyperref]{icml2026}

\newtheorem{theorem}{Theorem}

\newtheorem{definition}{Definition}%

\definecolor{mygray}{gray}{0.9}
\definecolor{myred}{rgb}{1.0, 0.0, 0.0}
\definecolor{mygreen}{rgb}{0.0, 1.0, 0.0}

\usepackage[textsize=tiny]{todonotes}

\icmltitlerunning{Unrewarded Exploration in Large Language Models Reveals Latent Learning from Psychology}

\begin{document}

\twocolumn[
  \icmltitle{Unrewarded Exploration in Large Language Models \\ Reveals Latent Learning from Psychology}
  \icmlsetsymbol{intern}{*}
  \begin{icmlauthorlist}
    \icmlauthor{Jian Xiong}{fudan,baidu,intern}
    \icmlauthor{Jingbo Zhou}{baidu}
    \icmlauthor{Zihan Zhou}{auburn}
    \icmlauthor{Yixiong Xiao}{baidu}
    \icmlauthor{Le Zhang}{baidu}
    \icmlauthor{Jingyong Ye}{fudan}
    \icmlauthor{Rui Qian}{fudan}
    \icmlauthor{Yang Zhou}{auburn}
    \icmlauthor{Dejing Dou}{fudan,bedi}
  \end{icmlauthorlist}

  \icmlaffiliation{fudan}{Fudan University, China}
  \icmlaffiliation{auburn}{Auburn University, USA}
  \icmlaffiliation{baidu}{Baidu Research, China}
  \icmlaffiliation{bedi}{BEDI Cloud, China}

  \icmlcorrespondingauthor{Jingbo Zhou}{zhoujingbo@outlook.com}
  \icmlcorrespondingauthor{Yang Zhou}{yangzhou@auburn.edu}
  \icmlcorrespondingauthor{Dejing Dou}{dejingdou@gmail.com}

  \vskip 0.1in
]
\printAffiliationsAndNotice{* Work done when intern at Baidu Research}

\begin{abstract}
Latent learning, classically theorized by Tolman, shows that biological agents (e.g., rats) can acquire internal representations of their environment without rewards, enabling rapid adaptation once rewards are introduced.
In contrast, from a cognitive science perspective, reward learning remains overly dependent on external feedback, limiting flexibility and generalization.
Although recent advances in the reasoning capabilities of large language models (LLMs), such as OpenAI-o1 and DeepSeek-R1, mark a significant breakthrough, these models still rely primarily on reward-centric reinforcement learning paradigms. Whether and how the well-established phenomenon of latent learning in psychology can inform or emerge within LLMs' training remains largely unexplored.
In this work, we present novel findings from our experiments that LLMs also exhibit the latent learning dynamics. 
During an initial phase of unrewarded exploration, LLMs display modest performance improvements, as this phase allows LLMs to organize task-relevant knowledge without being constrained by reward-driven biases, and performance is further enhanced once rewards are introduced.
LLMs post-trained under this two-stage exploration regime ultimately achieve higher competence than those post-trained with reward-based reinforcement learning throughout.
Beyond these empirical observations, we also provide theoretical analyses for our experiments explaining why unrewarded exploration yields performance gains, offering a mechanistic account of these dynamics.
Specifically, we conducted extensive experiments across multiple model families and diverse task domains to establish the existence of the latent learning dynamics in LLMs.
\end{abstract}

\section{Introduction}\label{sec1}
\begin{figure*}[ht]
    \centering
    \includegraphics[width=0.95\linewidth]{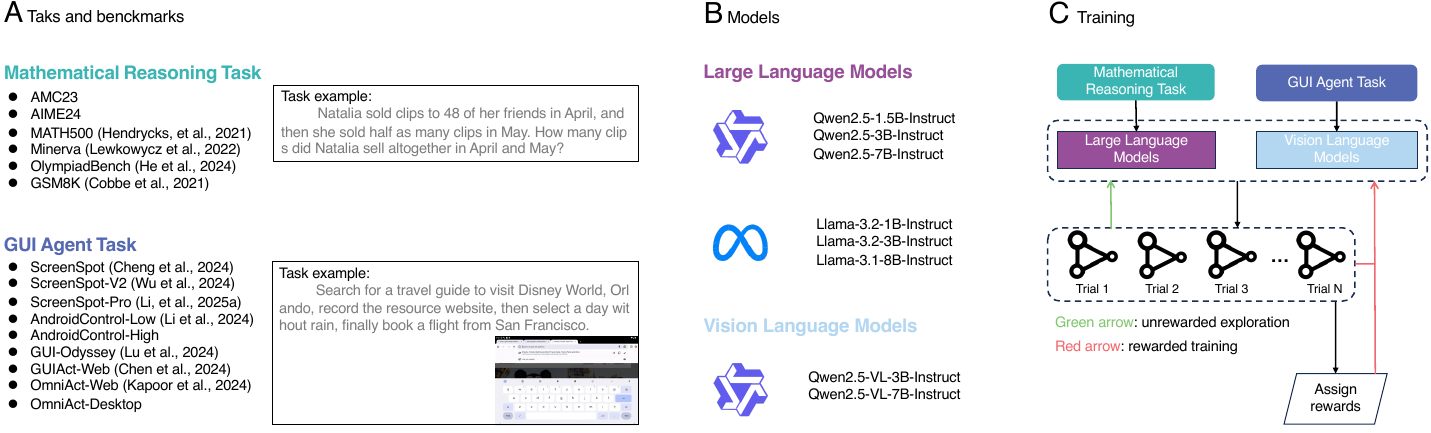}
    \caption{Overview of tasks and models we use in this work, and training paradigm. 
    \textbf{A}: Examples for the different tasks. 
    Experiment tasks include mathematical reasoning and GUI agent task. 
    \textbf{B}: Used LLMs include different series and parameter sizes.
    \textbf{C}: Our training process. For every task, questions are given to LLMs to generate a group of responses.
    The green arrow represents learning without rewards.
    The red arrow represents learning with rewards.}\label{figure1}
\end{figure*}

Learning without rewards has long challenged traditional views of behavior. 
In a series of classic studies, Blodgett and Tolman demonstrated that rats exploring mazes without food rewards were nonetheless able to acquire knowledge of the maze layout \citep{blodgett1929effect, tolman1948cognitive}. 
Once food was later introduced, their performance improved abruptly, revealing that substantial learning had taken place during the unrewarded exploration, a phenomenon that Tolman termed latent learning in psychology. These findings contradicted purely stimulus-response theories, which treat learning as reinforcement-driven association strengthening rather than the acquisition of internal representations of the environment. Consequently, the reward-centric learning often struggles to account for flexible and generalized behavior that emerges independently of direct reinforcement.

Recent advances in reinforcement learning for LLMs have achieved impressive reasoning performance, yet these models remain fundamentally reward-dependent rather than exhibiting the latent learning. Advanced reasoning models such as OpenAI-o1 and DeepSeek-R1 have demonstrated remarkable capabilities across diverse tasks \citep{bai2022training, achiam2023gpt, guo2025deepseek, yang2025qwen3, team2025kimi}. However, their learning frameworks remain deeply rooted in reward-centric paradigms, particularly reinforcement learning from outcome or process rewards and related optimization strategies\citep{schulman2017proximal, shao2024deepseekmath}. The prevailing assumption is that improvements in LLMs' capabilities must be explicitly guided by reward signals. Yet, if biological agents can acquire predictive world models without direct reinforcement, this raises a fundamental question: can LLMs also exhibit latent learning, as in Tolman’s theory that learning occurs even without rewards, improving performance even without the guidance of reward signals and revealing further gains once rewards are introduced? Addressing this question not only establishes a conceptual bridge between cognitive psychology and LLMs but also provides a basis for understanding how LLMs may internalize knowledge without rewards.

In this work, we provide novel empirical findings from our experiments that LLMs indeed display such latent learning as theorized by Tolman. 
We introduce an unrewarded exploration phase, analogous to Tolman’s latent learning, in which LLMs generate responses, omit response verification and rewards assignment, and subsequently learn directly from their own responses.
Unrewarded exploration allows LLMs to explore and organize task-relevant knowledge without being constrained by reward-driven biases, leading to more coherent internal representations even in the absence of rewards, hence improving performance.
Specifically, unrewarded exploration optimizes a KL-regularized policy matching objective that increases the likelihood of responses sampled from the current policy while constraining policy updates.
Under this setting, we observe that LLMs trained without rewards already exhibit measurable improvements in performance.
When rewards are subsequently introduced, these models demonstrate better performance gains compared to LLMs trained throughout with rewards. 
The resulting dynamics closely parallel Blodgett and Tolman’s experiments in rats~\citep{blodgett1929effect, tolman1948cognitive}, where learning happended during exploration without rewards and laid the groundwork for better performance improvements once food was provided.
Beyond these surprising empirical findings, we further conduct theoretical analyses focusing on the unrewarded exploration phase in our experiments, explaining how models learn in the absence of rewards, thereby gaining performance improvement.

To investigate whether latent learning arises in LLMs, we conduct experiments in two representative tasks for reinforcement learning for LLMs. 
First, we conduct experiments on mathematical reasoning task, a widely used task for reinforcement learning for LLMs, where performance improvements can be precisely measured through problem-solving accuracy~\citep{hendrycks2021measuring, cobbe2021training, he2024olympiadbench, lewkowycz2022solving}.
This domain provides a clean setting to examine whether performance can improve in the absence of reward signals and to quantify the extent of such improvement.
Second, we conduct experiments on Graphic User Interface (GUI) agent tasks, where LLMs are required to act like human to manipulate computers, mobile devices, etc., extending our experiments to interactive environments~\citep{cheng2024seeclick, wu2024atlas, li2025screenspot, li2024effects, lu2024gui, chen2024guicourse, kapoor2024omniact}. 
These tasks capture two distinct levels of difficulty: in low-level tasks, the LLMs is required to perform a single correct action (i.e., to ground which area to perform click), while in high-level tasks, the model must execute a sequence of actions (e.g., click, type, close and etc.), adapting iteratively to changing environments until the goal is achieved. 
This distinction allows us to test whether the latent learning happends under both minimal and extended forms of sequential decision making, bridging simple reasoning with more complex agent behaviors. Together, these two domains provide complementary perspectives. Mathematical reasoning isolates unrewarded exploration in a controlled, symbolic setting, while GUI agent tasks demonstrate its emergence in embodied, interactive contexts. This combination ensures that our findings are not confined to a single task domain, but instead generalize across both static and dynamic reasoning environments.

In summary, our contributions are as follows.
\begin{itemize} 
    \item We introduce an unrewarded exploration phase for LLMs training that yields performance gains even in the absence of rewards, with further improvements once rewards are introduced.
    \item We present theoretical analyses that provide mechanistic explanations for why LLMs can improve performance during unrewarded exploration.    
    \item We conduct extensive experiments across multiple model families, model parameter sizes, and diverse task domains to validate the existence of learning dynamics analogous to Tolman’s latent learning.
\end{itemize}

\section{Related Work}\label{sec2}
\subsection{Latent Learning in Psychology}\label{sec2.1}
The phenomenon of latent learning was first discovered in Blodgett’s experiments with rats, where animals exploring mazes without rewards later displayed abrupt performance improvements once food was introduced~\citep{blodgett1929effect}.
Tolman extended these findings and introduced the concept of cognitive maps, arguing that organisms form internal representations of the environment independent of rewards~\citep{tolman1948cognitive}.
Later studies, such as those by Spence and Lippitt, examined conditions under which latent learning manifests and highlighted both its robustness and its sensitivity to motivational states~\citep{spence1946experimental}.
Together, these studies established latent learning as a core concept in cognitive psychology, challenging conventional stimulus–response accounts of behavior.

\subsection{Reinforcement Learning Paradigm for Large Language Models}\label{sec2.2}
In artificial intelligence, reinforcement learning from human feedback (RLHF) has emerged as a central method for aligning and fine-tuning large language models.
Early work by~\citet{ziegler2019fine} demonstrated that human preference data combined with reinforcement learning substantially improved model alignment performance.
Most RLHF implementations rely on Proximal Policy Optimization (PPO)~\citep{schulman2017proximal}, which has become the universally acknowledged standard for policy optimization in many domains.
More recent studies have proposed improvements and alternatives to PPO, such as GRPO~\citep{shao2024deepseekmath}, which are attributed to the powerful capabilities of the DeepSeek-R1 model and the phenomenon of self-reflection~\citep{guo2025deepseek}, followed by a series of works~\citep{yu2025dapo, xiong2025aapo, zheng2025group}. 
These methods have been applied to domains such as mathematical reasoning, code generation and other agent tasks~\citep{luo2025gui, zeng2025simplerl,zhang2026omegausebuildinggeneralpurposegui}, but they all presuppose that reward signals are indispensable for driving performance improvement.

\subsection{Bridging Psychology and Artificial Intelligence}\label{sec2.3}
Prior work has extensively connected reinforcement learning to neural reward processing, particularly by relating temporal-difference learning to dopaminergic reward prediction errors~\citep{niv2009reinforcement, schultz1997neural} and has also developed reinforcement learning as a general computational framework beyond biological settings~\citep{doya2000reinforcement}. 
However, this line of research largely focuses on reward-driven learning and pays little attention to classical psychological phenomena that challenge the necessity of reinforcement, such as Tolman’s latent learning.
In cognitive psychology, latent learning demonstrates that biological agents (e.g. rats) can acquire task-relevant knowledge during unrewarded exploration, forming internal representations without rewards.
While related notions of internal representations have appeared in computational neuroscience and reinforcement learning, including the successor representation~\citep{dayan1993improving} and predictive map accounts~\citep{stachenfeld2017hippocampus}, these works do not examine whether unrewarded exploration alone yields measurable performance improvements in modern artificial intelligence systems (i.e., LLMs). 
To the best of our knowledge, our work provides the first explicit evidence of such latent learning dynamics in LLMs.
\begin{table*}[!h]
    \caption{Model pass@1 performance across various mathematical reasoning benchmarks, reported by the best average score.
    Rewarded: training with rewards.
    Unrewarded: exploration without rewards.
    We show the relative difference at the \textit{+ Unrewarded} row compared to the base model and the \textit{+ Unrewarded $\rightarrow$ Rewarded} row compared to the \textit{+ Rewarded} row.}
    \setlength{\tabcolsep}{3pt}
    \begin{adjustbox}{width=\textwidth, center} 
    \begin{tabular}{lccccccc}
    \toprule
    \textbf{Model} & \textbf{GSM8K} & \textbf{MATH500} & \textbf{Minerva} & \textbf{OlympiadBench} & \textbf{AIME24} & \textbf{AMC23} & \textbf{Avg}\\
    \midrule
    \multicolumn{8}{c}{\textit{Llama 1B Models}} \\
    \midrule
    Llama-3.2-1B-Instruct & 26.8 & 11.8 & 1.8 & 3.7 & 0.0 & 2.5 & \cellcolor{mygray} 7.8 \\
    + Unrewarded & 33.0 (\textcolor{myred}{+6.2}) & 15.0 (\textcolor{myred}{+3.2}) & 5.9 (\textcolor{myred}{+4.1}) & 3.6 (\textcolor{mygreen}{-0.1}) & 0.0 & 15.0 (\textcolor{myred}{+12.5}) & \cellcolor{mygray} 12.1 (\textcolor{myred}{+4.3}) \\
    + Rewarded & 56.1 & 33.2 & 7.0 & 9.2 & 10.0 & 17.5 & \cellcolor{mygray} 22.2\\
    + Unrewarded $\rightarrow$ Rewarded & 58.2 (\textcolor{myred}{+2.1}) & 37.8 (\textcolor{myred}{+4.6}) & 6.6 (\textcolor{mygreen}{-0.4}) & 9.9 (\textcolor{myred}{+0.7}) & 10.0 & 20.0 (\textcolor{myred}{+2.5}) & \cellcolor{mygray} \textbf{23.8} (\textcolor{myred}{+1.6}) \\
    \midrule 
    \multicolumn{8}{c}{\textit{Llama 3B Models}} \\
    \midrule
    Llama-3.2-3B-Instruct & 45.6 & 26.2 & 7.7 & 6.4 & 3.3 & 17.5 & \cellcolor{mygray} 17.8 \\
    + Unrewarded & 67.5 (\textcolor{myred}{+21.1}) & 36.6 (\textcolor{myred}{+10.4}) & 7.7 & 11.3 (\textcolor{myred}{+4.9}) & 6.7 (\textcolor{myred}{+3.4}) & 25.0 (\textcolor{myred}{+7.5}) & \cellcolor{mygray} 25.8 (\textcolor{myred}{+8.0}) \\
    + Rewarded & 82.0 & 53.4 & 23.2 & 18.7 & 20.0 & 37.5 & \cellcolor{mygray} 39.1 \\
    + Unrewarded $\rightarrow$ Rewarded & 83.5 (\textcolor{myred}{+1.5}) & 51.4 (\textcolor{mygreen}{-2.0}) & 16.5 (\textcolor{mygreen}{-6.7}) & 16.1 (\textcolor{mygreen}{-2.6}) & 26.7 (\textcolor{myred}{+6.7}) & 42.5 (\textcolor{myred}{+5.0}) & \cellcolor{mygray} \textbf{39.5} (\textcolor{myred}{+0.4}) \\
    \midrule 
    \multicolumn{8}{c}{\textit{Llama 8B Models}} \\
    \midrule
    Llama-3.1-8B-Instruct & 23.0 & 18.6 & 4.0 & 5.8 & 0.0 & 5.0 & \cellcolor{mygray} 9.4 \\
    + Unrewarded & 64.4 (\textcolor{myred}{+41.4}) & 26.8 (\textcolor{myred}{+10.2}) & 11.8 (\textcolor{myred}{+7.8}) & 5.8 & 0.0 & 25.0 (\textcolor{myred}{+20.0}) & \cellcolor{mygray} 22.3 (\textcolor{myred}{+12.9}) \\
    + Rewarded & 88.5 & 55.8 & 29.4 & 19.1 & 13.3 & 37.5 & \cellcolor{mygray} 40.6\\
    + Unrewarded $\rightarrow$ Rewarded & 88.5 & 53.8 (\textcolor{mygreen}{-2.0}) & 30.5 (\textcolor{myred}{+1.1}) & 22.5 (\textcolor{myred}{+3.2}) & 20.0 (\textcolor{myred}{+6.7}) & 32.5 (\textcolor{mygreen}{-5.0}) & \cellcolor{mygray} \textbf{41.3} (\textcolor{myred}{+0.7}) \\
    \midrule
    \multicolumn{8}{c}{\textit{Qwen 1.5B Models}} \\
    \midrule
    Qwen2.5-1.5B-Instruct & 68.7 & 47.6 & 16.5 & 18.1 & 0.0 & 20 & \cellcolor{mygray} 28.5 \\
    + Unrewarded & 68.2 (\textcolor{mygreen}{-0.5}) & 48.0 (\textcolor{myred}{+0.4}) & 17.3 (\textcolor{myred}{+0.8}) & 19.3 (\textcolor{myred}{+1.2}) & 0.0 & 32.5 (\textcolor{myred}{+12.5}) & \cellcolor{mygray} 30.9 (\textcolor{myred}{+2.4}) \\
    + Rewarded & 71.0 & 54.4 & 16.5 & 20.1 & 0.0 & 32.5 & \cellcolor{mygray} 32.4 \\
    + Unrewarded $\rightarrow$ Rewarded & 72.2 (\textcolor{myred}{+1.2}) & 54.2 (\textcolor{mygreen}{-0.2}) & 19.5 (\textcolor{myred}{+3.0}) & 19.4 (\textcolor{mygreen}{-0.7}) & 10.0 (\textcolor{myred}{+10.0}) & 37.5 (\textcolor{myred}{+5.0}) & \cellcolor{mygray} \textbf{35.5} (\textcolor{myred}{+3.1}) \\
    \midrule
    \multicolumn{8}{c}{\textit{Qwen 3B Models}} \\
    \midrule
    Qwen2.5-3B-Instruct & 84.9 & 63.0 & 24.3 & 24.4 & 6.7 & 37.5 & \cellcolor{mygray} 40.1 \\
    + Unrewarded & 85.7 (\textcolor{myred}{+0.8}) & 61.4 (\textcolor{mygreen}{-1.6}) & 25.7 (\textcolor{myred}{+1.4}) & 27.0 (\textcolor{myred}{+2.6}) & 6.7 & 50.0 (\textcolor{myred}{+12.5}) & \cellcolor{mygray} 42.8 (\textcolor{myred}{+2.7}) \\
    + Rewarded & 84.7 & 65.8 & 25.7 & 28.1 & 13.3 & 45.0 & \cellcolor{mygray} 43.8 \\
    + Unrewarded $\rightarrow$ Rewarded & 85.6 (\textcolor{myred}{+0.9}) & 67.6 (\textcolor{myred}{+1.8}) & 25.7 & 29.3 (\textcolor{myred}{+1.2}) & 10.0 (\textcolor{mygreen}{-3.3}) & 55.0 (\textcolor{myred}{+10.0}) & \cellcolor{mygray} \textbf{45.5} (\textcolor{myred}{+1.7}) \\
    \midrule
    \multicolumn{8}{c}{\textit{Qwen 7B Models}} \\
    \midrule
    Qwen2.5-7B-Instruct & 91.6 & 73.6 & 37.9 & 36.7 & 6.7 & 47.5 & \cellcolor{mygray} 49.0 \\
    + Unrewarded & 91.0 (\textcolor{mygreen}{-0.6}) & 77.4 (\textcolor{myred}{+3.8}) & 37.1 (\textcolor{mygreen}{-0.8}) & 36.4 (\textcolor{mygreen}{-0.3}) & 20.0 (\textcolor{myred}{+13.3}) & 50.0 (\textcolor{myred}{+2.5}) & \cellcolor{mygray} 52.0 (\textcolor{myred}{+3.0}) \\
    + Rewarded & 91.7 & 76.6 & 36.8 & 43.1 & 23.3 & 57.5 & \cellcolor{mygray} 54.8 \\
    + Unrewarded $\rightarrow$ Rewarded & 91.7 & 77.6 (\textcolor{myred}{+1.0}) & 44.1 (\textcolor{myred}{+7.3}) & 40.1 (\textcolor{mygreen}{-3.0}) & 20.0 (\textcolor{mygreen}{-3.3}) & 60.0 (\textcolor{myred}{+2.5}) & \cellcolor{mygray} \textbf{55.6} (\textcolor{myred}{+0.8}) \\
    \bottomrule
    \end{tabular}
    \end{adjustbox}
    \label{tab:1}
\end{table*}
\section{Method}\label{sec3}
\subsection{Rewarded Training}\label{sec3.1}
We adopt the training objective proposed by~\citet{shao2024deepseekmath} as our reinforcement learning objective with rewards:
\begin{equation}\notag
\scalebox{0.76}{$
\begin{multlined}
\mathcal{J}_\mathrm{Rewarded}(\theta)
=\mathbb{E}_{(q, a) \sim \mathcal{D},\ \left\{o_i\right\}_{i=1}^G \sim \pi_{\mathrm{old}}(\cdot \mid q)}\frac{1}{G} \sum_{i=1}^G \frac{1}{\left|o_i\right|} \sum_{t=1}^{\left|o_i\right|}\\
\left\{\min \left[r_{i, t}(\theta)\hat{A}_{i,t}, \operatorname{clip}\left(r_{i, t}(\theta), 1-\varepsilon, 1+\varepsilon\right) \hat{A}_{i,t}\right]-\beta \mathbb{D}_{K L}\left[\pi_\theta| | \pi_{r e f}\right]\right\},
\end{multlined}$}
\end{equation}
where $q$, $a$, $\mathcal{D}$, $o$, $G$, $\pi_{\mathrm{old}}$, $\pi_{r e f}$ and $\pi_\theta$ represent question, answer, dataset, model response, group size, old policy model, reference model, and current policy model, respectively.
The clip operation is designed to constrain the input variable within the predefined closed interval $[1-\varepsilon, 1+\varepsilon]$, and $\varepsilon$ is a hyper-parameter introduced in PPO~\citep{schulman2017proximal} for stabilizing training. 
$r_{i, t}(\theta)$ is computed as $r_{i, t}(\theta)=\frac{\pi_\theta(o_{i,t} \mid q,o_{i,<t})}{\pi_{\mathrm{old}}(o_{i,t} \mid q,o_{i,<t})}$.
The advantage $\hat{A}_{i,t}$ is computed as $\hat{A}_{i,t}=\frac{r_i-\operatorname{mean}\{R\}_{i=1}^G}{\operatorname{std}\{R\}_{i=1}^G}$, where $r_i$ is the reward assigned by the reward rules for the response $o_i$.
And the KL-divergence $\mathbb{D}_{K L}\left[\pi_\theta| | \pi_{r e f}\right]$ is computed by
\begin{equation}\notag
\scalebox{0.82}{$\displaystyle
\begin{aligned}
\mathbb{D}_{K L}\left[\pi_\theta| | \pi_{r e f}\right]=\frac{\pi_{ref}(o_{i,t} \mid q,o_{i,<t})}{\pi_\theta(o_{i,t} \mid q,o_{i,<t})}-\log\frac{\pi_{ref}(o_{i,t} \mid q,o_{i,<t})}{\pi_\theta(o_{i,t} \mid q,o_{i,<t})}-1.
\end{aligned}
$}
\end{equation}

\subsection{Unrewarded Exploration}\label{sec3.2}
In RL training, rewards are used to calculate the advantage value for each response (as expressed in Section~\ref{sec3.1}), when no rewards are assigned for all responses, each response should be treated equally.
During this stage, we adopt the same objective as training with rewards but set the advantage value to 1. 
That means our training objective during unrewarded exploration is
\begin{equation}\notag
\scalebox{0.88}{$
\begin{multlined}
\mathcal{J}_\mathrm{Unrewarded}(\theta)
=\mathbb{E}_{(q, a) \sim \mathcal{D},\ \left\{o_i\right\}_{i=1}^G \sim \pi_{\mathrm{old}}(\cdot \mid q)}\frac{1}{G} \sum_{i=1}^G \frac{1}{\left|o_i\right|} \sum_{t=1}^{\left|o_i\right|}
\\\left\{\min \left[r_{i, t}(\theta), \operatorname{clip}\left(r_{i, t}(\theta), 1-\varepsilon, 1+\varepsilon\right) \right]-\beta \mathbb{D}_{K L}\left[\pi_\theta| | \pi_{r e f}\right]\right\},
\end{multlined}
$}
\end{equation}
where all symbols in the training objective remain the same as those in Section~\ref{sec3.1}.  
\subsection{Theoretical Analyses}\label{sec4.5}
Here, we provide theoretical analyses for our experiment settings explaining why LLMs exhibit improved performance during unrewarded exploration.
\begin{definition}[Probability Simplex]
For a finite vocabulary $\mathcal V$ of size $|\mathcal V|$, the probability simplex $\Delta^{|\mathcal V|-1}$ is defined as
\begin{equation}\notag
    \Delta^{|\mathcal V|-1}=\Big\{\pi:\mathcal V\to [0,1]\;\Big|\;\sum_{a\in\mathcal V}\pi(a)=1\Big\}.
\end{equation}
In particular, for each token state $s=(x,y_{<t})$, the conditional distribution $\pi(\cdot\mid s)$ is required to belong to $\Delta^{|\mathcal V|-1}$:
\begin{equation}\notag
    \pi(\cdot\mid s)\in \Delta^{|\mathcal V|-1},\qquad \sum_{a\in\mathcal V}\pi(a\mid s)=1.
\end{equation}
\end{definition}
For GUI agent task, LLMs are required to response action type from [`complete', `close/delete', `press\_home', `click', `press\_back', `type', `select', `scroll', `enter'], which means that the action type space for response is sparse.
Here, we provide the theoretical analysis on the performance gain after unrewarded exploration in sparse setting in Theorem~\ref{theorem:1}, also see detailed proof in Appendix~\ref{appendix:proof}.
\begin{theorem}[Performance improvement without rewards in sparse setting]\label{theorem:1}
Consider an autoregressive language model policy $\pi_\theta$ that defines, for each input prompt $x$ and token prefix $y_{<t}$, a conditional distribution\(\pi_\theta(y_t \mid x,y_{<t})\) over the next token $y_t$. The model is trained with Group Relative Policy Optimization where the per-token advantage is fixed to $A_{i,t}\equiv 1$, ratio clipping is applied with parameter $\varepsilon>0$, and a KL trust-region penalty of strength $\beta>0$ is included. Let the reference policy be $\pi_{\text{ref}}(y_t \mid x,y_{<t})$, and let the proposal distribution be $\pi_{\text{old}}(y_t \mid x,y_{<t})$. Assume that the likelihood ratio \(h(x,y_{<t},y_t) \;=\; \frac{\pi_{\text{old}}(y_t \mid x,y_{<t})}{\pi_{\text{ref}}(y_t \mid x,y_{<t})}\) is nondecreasing in the latent utility $u^\star(x,y_{<t},y_t)$ for every input $x$, every prefix $y_{<t}$, and every token $y_t$.

Then the GRPO update produces an updated conditional policy of the closed form
\begin{equation}\notag
\scalebox{0.75}{$\displaystyle
\begin{aligned}
    \pi^\star_\theta(y_t \mid x,y_{<t}) \;=\; \min\!\Big((1+\varepsilon)\,\pi_{\text{old}}(y_t \mid x,y_{<t}),\ \ \tau(x,y_{<t})\,\pi_{\text{ref}}(y_t \mid x,y_{<t})\Big),
\end{aligned}
$}
\end{equation}
where $\tau(x,y_{<t})>0$ is the unique normalizing constant satisfying \(
\sum_{y_t \in \mathcal V} \pi^\star_\theta(y_t \mid x,y_{<t}) = 1.
\) Moreover, under the monotone-likelihood-ratio condition, the expected performance
\begin{equation}\notag
    J(\theta) \;=\; \mathbb E_{x \sim \mathcal D}\,\mathbb E_{y \sim \pi_\theta(\cdot \mid x)} \Bigg[\sum_{t=1}^{|y|} u^\star(x,y_{<t},y_t)\Bigg]
\end{equation}

satisfies \(
J(\pi^\star_\theta) \;\ge\; J(\pi_{\text{ref}}).
\) Thus even with no scalar reward signal and with constant advantage, ratio-clipped GRPO with KL regularization yields monotone improvement in the expected utility.
\end{theorem}

For mathematical reasoning task, LLMs solve the problems like human, which means that the response space is the whole vocabulary of LLMs.
Hence, we further provide the theoretical analysis on performance gain after unrewarded exploration in continuous setting in Theorem~\ref{theorem:2}, also see detailed proof in Appendix~\ref{appendix:proof}.
\begin{theorem}[Performance improvement without rewards in continuous setting]\label{theorem:2}
Fix an input distribution $\mathcal D$, an autoregressive language model policy $\pi_\theta(y_t\mid x,y_{<t})$ that admits a density with respect to a common $\sigma$-finite base measure on the token space, a reference policy $\pi_{\text{ref}}(y_t\mid x,y_{<t})$ with the same property, and a proposal (teacher) policy $\pi_{\mathrm{prop}}(y_t\mid x,y_{<t})$. Train with Group Relative Policy Optimization in which the per-token advantage is fixed to $A\equiv 1$, the ratio is clipped with parameter $\varepsilon>0$ so that the per-token factor becomes $\min(\rho,1+\varepsilon)$, and a KL trust-region penalty of strength $\beta>0$ is included. Assume that for every token state $(x,y_{<t})$ the likelihood ratio $h(x,y_{<t},y_t)=\pi_{\mathrm{prop}}(y_t\mid x,y_{<t})/\pi_{\text{ref}}(y_t\mid x,y_{<t})$ is comonotone with the latent token utility $u^\star(x,y_{<t},y_t)$ under the base distribution $\pi_{\text{ref}}(\cdot\mid x,y_{<t})$, meaning that $(u^\star(\cdot)-u^\star(\cdot'))(h(\cdot)-h(\cdot'))\ge 0$ almost everywhere. Then the statewise maximizer of the GRPO surrogate has the water-filling form
\begin{equation}\notag
\scalebox{0.76}{$
\begin{gathered}
    \pi^\star_\theta(y_t\mid x,y_{<t})=\min\big((1+\varepsilon)\,\pi_{\mathrm{prop}}(y_t\mid x,y_{<t}),\ \tau_{x,y_{<t}}\,\pi_{\text{ref}}(y_t\mid x,y_{<t})\big),\\
    \int \pi^\star_\theta(y_t\mid x,y_{<t})\,\mathrm d\nu(y_t)=1,
\end{gathered}
$}
\end{equation}
for a unique normalizing constant $\tau_{x,y_{<t}}>0$. Moreover, the expected performance
\begin{equation}\notag
\scalebox{0.9}{$
\begin{multlined}
J(\theta)=\mathbb E_{x\sim\mathcal D}\,\mathbb E_{y\sim \pi_\theta(\cdot\mid x)}\Big[\sum_{t=1}^{|y|}u^\star(x,y_{<t},y_t)\Big]
\\=\mathbb E_{x,y_{<t}}\int \pi_\theta(y_t\mid x,y_{<t})\,u^\star(x,y_{<t},y_t)\,\mathrm d\nu(y_t)
\end{multlined}$}
\end{equation}
satisfies $J(\pi^\star_\theta)\ge J(\pi_{\text{ref}})$. Hence, even without scalar rewards and with constant advantage, ratio-clipped GRPO with KL regularization yields monotone improvement in expected latent utility in continuous output spaces.
\end{theorem}
The theoretical analyses above simultaneously examines the feasibility of enhancing LLMs' performance in both discrete spaces (where the action space available to LLMs is predefined in GUI agent tasks) and continuous spaces (where no constraints are imposed on LLMs' response in mathematical reasoning tasks) during unrewarded exploration.
\begin{table}[!t]
    \caption{Model pass@1 performance across various mathematical reasoning benchmarks, reported by best accuracy on each benchmark.
    Rewarded: training with rewards.
    Unrewarded: exploration without rewards.
    We show the relative difference at the \textit{+ Unrewarded} row compared to the base model and the \textit{+ Unrewarded $\rightarrow$ Rewarded} row compared to the \textit{+ Rewarded} row.}
    \small
    \setlength{\tabcolsep}{3pt}
    \begin{adjustbox}{width=0.5\textwidth, center} 
    \begin{tabular}{lccccccc}
    \toprule
    \textbf{Model} & \textbf{GSM8K} & \textbf{MATH500} & \textbf{Minerva} & \textbf{OlympiadBench} & \textbf{AIME24} & \textbf{AMC23} \\
    \midrule
    \multicolumn{7}{c}{\textit{Llama 1B Models}} \\
    \midrule
    Llama-3.2-1B-Instruct & 26.8 & 11.8 & 1.8 & 3.7 & 0.0 & 2.5 \\
    + Unrewarded & 33.1 (\textcolor{myred}{+4.3}) & 18.6 (\textcolor{myred}{+6.8}) & 6.2 (\textcolor{myred}{+4.4}) & 5.5 (\textcolor{myred}{+1.8}) & 3.3 (\textcolor{myred}{+3.3}) & 15.0 (\textcolor{myred}{+12.5}) \\
    + Rewarded & 57.0 & 37.0 & 9.2 & \textbf{11.7} & 10.0 & \textbf{25.0} \\
    + Unrewarded $\rightarrow$ Rewarded & \textbf{59.4} (\textcolor{myred}{+2.4}) & \textbf{38.8} (\textcolor{myred}{+1.8}) & \textbf{10.3} (\textcolor{myred}{+1.1}) & 11.6 (\textcolor{mygreen}{-0.1}) & \textbf{13.3} (\textcolor{myred}{+3.3}) & \textbf{25.0} \\
    \midrule
    \multicolumn{7}{c}{\textit{Llama 3B Models}} \\
    \midrule
    Llama-3.2-3B-Instruct & 45.6 & 26.2 & 7.7 & 6.4 & 3.3 & 17.5 \\
    + Unrewarded & 67.5 (\textcolor{myred}{+21.9}) & 40.8 (\textcolor{myred}{+14.6}) & 11.0 (\textcolor{myred}{+3.3}) & 12.1 (\textcolor{myred}{+5.7}) & 10.0 (\textcolor{myred}{+6.7}) & 25.0 (\textcolor{myred}{+7.5}) \\
    + Rewarded & \textbf{84.2} & 53.8 & \textbf{23.2} & 19.7 & 23.3 & 40.0 \\
    + Unrewarded $\rightarrow$ Rewarded & 83.7 (\textcolor{mygreen}{-0.5}) & \textbf{55.2} (\textcolor{myred}{+1.4}) & 21.3 (\textcolor{mygreen}{-1.9}) & \textbf{19.9} (\textcolor{myred}{+0.2}) & \textbf{26.7} (\textcolor{myred}{+3.4}) & \textbf{42.5} (\textcolor{myred}{+2.5}) \\
    \midrule
    \multicolumn{7}{c}{\textit{Llama 8B Models}} \\
    \midrule
    Llama-3.1-8B-Instruct & 23.0 & 18.6 & 4.0 & 5.8 & 0.0 & 5.0 \\
    + Unrewarded & 64.5 (\textcolor{myred}{+41.5}) & 31.2 (\textcolor{myred}{+12.6}) & 12.9 (\textcolor{myred}{+8.9}) & 8.6 (\textcolor{myred}{+2.8}) & 6.7 (\textcolor{myred}{+6.7}) & 25.0 (\textcolor{myred}{+20.0}) \\
    + Rewarded & \textbf{90.1} & 59.2 & 31.2 & 22.5 & \textbf{20.0} & \textbf{40.0} \\
    + Unrewarded $\rightarrow$ Rewarded & 89.6 (\textcolor{mygreen}{-0.5}) & \textbf{59.8} (\textcolor{myred}{+0.6}) & \textbf{33.1} (\textcolor{myred}{+1.9}) & \textbf{23.1} (\textcolor{myred}{+0.6}) & \textbf{20.0} & \textbf{40.0} \\
    \midrule
    \multicolumn{7}{c}{\textit{Qwen 1.5B Models}} \\
    \midrule
    Qwen2.5-1.5B-Instruct & 68.7 & 47.6 & 16.5 & 18.1 & 0.0 & 20  \\
    + Unrewarded & 70.1 (\textcolor{myred}{+1.4}) & 51.6 (\textcolor{myred}{+4.0}) & 17.3 (\textcolor{myred}{+0.8}) & 19.3 (\textcolor{myred}{+1.2}) & 0.0 & 32.5 (\textcolor{myred}{+12.5}) \\
    + Rewarded & 73.5 & 54.8 & 19.1 & 20.3 & 3.3 & 32.5 \\
    + Unrewarded $\rightarrow$ Rewarded & \textbf{75.3} (\textcolor{myred}{+1.8}) & \textbf{58.4} (\textcolor{myred}{+3.6}) & \textbf{19.5} (\textcolor{myred}{+0.4}) & \textbf{22.1} (\textcolor{myred}{+1.8}) & \textbf{10.0} (\textcolor{myred}{+6.7}) & \textbf{40.0} (\textcolor{myred}{+7.5}) \\
    \midrule
    \multicolumn{7}{c}{\textit{Qwen 3B Models}} \\
    \midrule
    Qwen2.5-3B-Instruct & 84.9 & 63.0 & 24.3 & 24.4 & 6.7 & 37.5 \\
    + Unrewarded & 85.7 (\textcolor{myred}{+0.8}) & 66.8 (\textcolor{myred}{+3.8}) & 26.8 (\textcolor{myred}{+2.5}) & 28.3 (\textcolor{myred}{+3.9}) & 6.7 & 50.0 (\textcolor{myred}{+12.5}) \\
    + Rewarded & 86.6 & 68.4 & 28.7 & 31.9 & \textbf{16.7} & 50.0 \\
    + Unrewarded $\rightarrow$ Rewarded & \textbf{87.6} (\textcolor{myred}{+1.0}) & \textbf{70.2} (\textcolor{myred}{+1.8}) & \textbf{31.2} (\textcolor{myred}{+2.5}) & \textbf{34.8} (\textcolor{myred}{+2.9}) & 13.3 (\textcolor{mygreen}{-3.4}) & \textbf{55.0} (\textcolor{myred}{+5.0}) \\
    \midrule
    \multicolumn{7}{c}{\textit{Qwen 7B Models}} \\
    \midrule
    Qwen2.5-7B-Instruct & 91.6 & 73.6 & 37.9 & 36.7 & 6.7 & 47.5 \\
    + Unrewarded & 92.1 (\textcolor{myred}{+0.5}) & 77.8 (\textcolor{myred}{+4.2}) & 42.3 (\textcolor{myred}{+4.4}) & 38.8 (\textcolor{myred}{+2.1}) & 20.0 (\textcolor{myred}{+13.3}) & 57.5 (\textcolor{myred}{+10.0}) \\
    + Rewarded & \textbf{93.3} & \textbf{79.6} & 38.2 & \textbf{43.1} & \textbf{23.3} & 62.5 \\
    + Unrewarded $\rightarrow$ Rewarded & 92.9 (\textcolor{mygreen}{-0.4}) & 79.2 (\textcolor{mygreen}{-0.4}) & \textbf{44.1} (\textcolor{myred}{+5.9}) & 41.6 (\textcolor{mygreen}{-1.5}) & \textbf{23.3} & \textbf{67.5} (\textcolor{myred}{+5.0}) \\
    \bottomrule
    \end{tabular}
    \end{adjustbox}
    \label{tab:2}
\end{table}
\begin{table}[!b]
    \caption{Model performance on Screenspot-Pro.
    Rewarded: training with rewards.
    Unrewarded: exploration without rewards.
    Text: text target.
    Icon: icon target.
    We show the relative difference at the \textit{+ Unrewarded} row compared to the base model and the \textit{+ Unrewarded $\rightarrow$ Rewarded} row compared to the \textit{+ Rewarded} row.}
    \setlength{\tabcolsep}{3pt}
    \begin{adjustbox}{width=0.5\textwidth, center} 
    \begin{tabular}{lccccccccccccccc}
    \toprule
    \multirow{2}{*}{\textbf{Model}} & 
    \multicolumn{2}{c}{\textbf{Development}} & 
    \multicolumn{2}{c}{\textbf{Creative}} & 
    \multicolumn{2}{c}{\textbf{CAD}} & 
    \multicolumn{2}{c}{\textbf{Scientific}} & 
    \multicolumn{2}{c}{\textbf{Office}} & 
    \multicolumn{2}{c}{\textbf{OS}} \\
    \cmidrule(lr){2-3} \cmidrule(lr){4-5} \cmidrule(lr){6-7} \cmidrule(lr){8-9} \cmidrule(lr){10-11} \cmidrule(lr){12-13}
    & Text & Icon & Text & Icon & Text & Icon & Text & Icon & Text & Icon & Text & Icon & \\
    \midrule
    \multicolumn{14}{c}{\textit{Qwen 3B Models}} \\
    \midrule
    Qwen2.5-VL-3B-Instruct & 14.3 & 0.7 & 23.7 & 2.1 & 11.2 & 3.1 & 34.7 & 8.2 & 21.5 & 3.8 & 14.0 & 2.3 \\
    + Unrewarded & \makecell{22.1\\ \textcolor{myred}{(+7.8)}} & \makecell{2.8\\ \textcolor{myred}{(+2.1)}} & \makecell{30.3\\ \textcolor{myred}{(+6.6)}} & \makecell{\underline{4.9}\\ \textcolor{myred}{(+2.8)}} & \makecell{15.2\\ \textcolor{myred}{(+4.0)}} & \makecell{\underline{7.8}\\ \textcolor{myred}{(+4.7)}} & \makecell{46.5\\ \textcolor{myred}{(+11.8)}} & \makecell{\underline{11.8}\\ \textcolor{myred}{(+3.6)}} & \makecell{33.3\\ \textcolor{myred}{(+11.8)}} & \makecell{11.3\\ \textcolor{myred}{(+7.5)}} & \makecell{19.6\\ \textcolor{myred}{(+5.6)}} & \makecell{\underline{6.7}\\ \textcolor{myred}{(+4.4)}} \\
    + Rewarded & \underline{33.8} & \underline{4.8} & \underline{41.4} & 2.8 & \underline{26.9} & \underline{7.8} & \textbf{61.8} & \textbf{17.3} & \textbf{53.7} & \underline{17.0} & \underline{27.1} & \underline{6.7} \\
    + Unrewarded $\rightarrow$ Rewarded & \makecell{\textbf{35.1}\\ \textcolor{myred}{(+1.3)}} & \makecell{\textbf{6.2}\\ \textcolor{myred}{(+1.4)}} & \makecell{\textbf{41.9}\\ \textcolor{myred}{(+0.5)}} & \makecell{\textbf{5.6}\\ \textcolor{myred}{(+2.8)}} & \makecell{\textbf{33.0}\\ \textcolor{myred}{(+6.1)}} & \makecell{\textbf{10.9}\\ \textcolor{myred}{(+3.1)}} & \makecell{\underline{60.4}\\ \textcolor{mygreen}{(-1.4)}} & \textbf{17.3} & \makecell{\underline{50.9}\\ \textcolor{mygreen}{(-2.8)}} & \makecell{\textbf{18.9}\\ \textcolor{myred}{(+1.9)}} & \makecell{\textbf{39.0}\\ \textcolor{myred}{(+11.9)}} & \makecell{\textbf{7.9}\\ \textcolor{myred}{(+1.2)}} \\  
    \midrule
    \multicolumn{14}{c}{\textit{Qwen 7B Models}} \\
    \midrule
    Qwen2.5-VL-7B-Instruct & 30.5 & 2.1 & 22.7 & 3.5 & 10.7 & 3.1 & 39.6 & 6.4 & 36.2 & 11.3 & 29.9 & 5.6 \\
    + Unrewarded & \makecell{40.9\\ \textcolor{myred}{(+10.4)}} & \makecell{2.8\\ \textcolor{myred}{(+0.7)}} & \makecell{31.3\\ \textcolor{myred}{(+8.6)}} & \makecell{\underline{6.3}\\ \textcolor{myred}{(+2.8)}} & \makecell{17.8\\ \textcolor{myred}{(+7.1)}} & 3.1 & \makecell{45.8\\ \textcolor{myred}{(+6.2)}} & \makecell{9.1\\ \textcolor{myred}{(+2.7)}} & \makecell{44.6\\ \textcolor{myred}{(+8.4)}} & \makecell{13.2\\ \textcolor{myred}{(+1.9)}} & \makecell{\underline{33.7}\\ \textcolor{myred}{(+3.8)}} & \makecell{13.5\\ \textcolor{myred}{(+7.9)}} \\
    + Rewarded & \textbf{50.7} & \textbf{4.8} & \underline{38.9} & \textbf{9.1} & \underline{24.9} & \textbf{6.3} & \underline{55.6} & \underline{11.8} & \underline{58.8} & \textbf{26.4} & \textbf{41.1} & \textbf{16.9} \\
    + Unrewarded $\rightarrow$ Rewarded & \makecell{\underline{50.0}\\ \textcolor{mygreen}{(-0.7)}} & \makecell{\underline{3.5}\\ \textcolor{mygreen}{(-1.3)}} & \makecell{\textbf{41.9}\\ \textcolor{myred}{(+3.0)}} & \textbf{9.1} & \makecell{\textbf{31.0}\\ \textcolor{myred}{(+6.1)}} & \makecell{\underline{4.7}\\ \textcolor{mygreen}{(-1.6)}} & \makecell{\textbf{62.5}\\ \textcolor{myred}{(+6.9)}} & \makecell{\textbf{13.6}\\ \textcolor{myred}{(+2.8)}} & \makecell{\textbf{64.4}\\ \textcolor{myred}{(+5.6)}} & \makecell{\underline{20.8}\\ \textcolor{mygreen}{(-5.6)}} & \textbf{41.1} & \makecell{\underline{15.7}\\ \textcolor{mygreen}{(-1.2)}}\\
    \bottomrule
    \end{tabular}
    \end{adjustbox}
    \label{tab:3}
\end{table}
\section{Experiments}\label{sec4}
\subsection{Models for Experiments}\label{sec4.1}
The models we use in all experiments include the Qwen series and the LLaMA series, both of which are open-source models and can be downloaded from the Huggingface website (\href{https://huggingface.co}{https://huggingface.co/}).
The Qwen series includes models Qwen2.5-1.5B-Instruct, Qwen2.5-3B-Instruct, Qwen2.5-7B-Instruct, Qwen2.5-VL-3B-Instruct and Qwen2.5-VL-3B-Instruct while the LLaMA series includes models Llama-3.2-1B-Instruct, Llama-3.2-3B-Instruct, Llama-3.1-8B-Instruct.
These choices span multiple model families and a broad range of model scales, allowing us to examine whether LLMs show such latent learning with Tolman across both series and parameter sizes.

\subsection{Datasets for experiments}\label{sec4.2}
\subsubsection{Mathematical reasoning}\label{sec4.2.1}
We used the training data from SimpleRL-Zoo-Data~\citep{zeng2025simplerl} collection for training, which was constructed from the training data of GSM8K~\citep{cobbe2021training} and MATH~\citep{hendrycks2021measuring}. 
For the Qwen series models, we use the simplelr\_qwen\_level3to5 dataset, and for the Llama series models, we use the simplelr\_abel\_level3to5 dataset.
The only difference between the two datasets is the format requirements for the output in prompts.
The simplelr\_qwen\_level3to5 dataset requires the model to output the final answer within the \textbackslash box\{\} tag, while the simplelr\_abel\_level3to5 dataset has no output format requirements. However, the mathematical reasoning problems and ground truth in both training data are the same. 
We evaluate the trained models on AMC23, AIME24, GSM8K~\citep{cobbe2021training}, MATH500~\citep{hendrycks2021measuring}, OlympiadBench~\citep{he2024olympiadbench} and Minerva Math~\citep{lewkowycz2022solving} benchmarks.
\begin{table}[!b]
    \caption{Model performance on Screenspot.
    Rewarded: training with rewards.
    Unrewarded: exploration without rewards.
    Text: text target.
    Icon/Widget: icon/widget target.
    We show the relative difference at the \textit{+ Unrewarded} row compared to the base model and the \textit{+ Unrewarded $\rightarrow$ Rewarded} row compared to the \textit{+ Rewarded} row.}
    \setlength{\tabcolsep}{3pt}
    \begin{adjustbox}{width=0.5\textwidth, center}
    \begin{tabular}{lccccccc}
    \toprule
    \multirow{2}{*}{\textbf{Model}} & 
    \multicolumn{2}{c}{\textbf{Mobile}} & 
    \multicolumn{2}{c}{\textbf{Web}} & 
    \multicolumn{2}{c}{\textbf{Desktop}}&
    \multirow{2}{*}{\textbf{Avg}} \\
    \cmidrule(lr){2-3} \cmidrule(lr){4-5} \cmidrule(lr){6-7}
    & Text & Icon/Widget & Text & Icon/Widget & Text & Icon/Widget &\\
    \midrule
    \multicolumn{8}{c}{\textit{Qwen 3B Models}} \\
    \midrule
    Qwen2.5-VL-3B-Instruct & 73.5 & 50.5 & 71.1 & 31.4 & 88.3 & 58.1 & \cellcolor{mygray} 62.2 \\
    + Unrewarded & \makecell{88.7\\ \textcolor{myred}{(+15.2)}} & \makecell{64.6\\ \textcolor{myred}{(+14.1)}} & \makecell{78.3\\ \textcolor{myred}{(+7.2)}} & \makecell{60.2\\ \textcolor{myred}{(+28.8)}} & \makecell{87.1\\ \textcolor{mygreen}{(-1.2)}} & \makecell{52.1\\ \textcolor{mygreen}{(-6.0)}} & \cellcolor{mygray} \makecell{71.8\\ \textcolor{myred}{(+9.6)}} \\
    + Rewarded & \textbf{96.3} & \textbf{76.4} & \underline{87.4} & \underline{66.0} & \textbf{93.8} & \underline{57.1} & \cellcolor{mygray} \underline{79.5} \\
    + Unrewarded $\rightarrow$ Rewarded & \makecell{\underline{94.9}\\ \textcolor{mygreen}{(-1.4)}} & \makecell{\underline{75.6}\\ \textcolor{mygreen}{(-0.8)}} & \makecell{\textbf{87.8}\\ \textcolor{myred}{(+0.4)}} & \makecell{\textbf{68.5}\\ \textcolor{myred}{(+2.5)}} & \makecell{\underline{92.3}\\ \textcolor{mygreen}{(-1.5)}} & \makecell{\textbf{59.3}\\ \textcolor{myred}{(+2.2)}} & \cellcolor{mygray} \makecell{\textbf{79.7}\\ \textcolor{myred}{(+0.2)}} \\
    \midrule
    \multicolumn{8}{c}{\textit{Qwen 7B Models}} \\
    \midrule
    Qwen2.5-VL-7B-Instruct & 93.8 & 70.3 & 87.0 & 68.0 & 91.2 & 60.0 & \cellcolor{mygray} 78.4 \\
    + Unrewarded & \makecell{\underline{95.2}\\ \textcolor{myred}{(+1.4)}} & \makecell{64.2\\ \textcolor{mygreen}{(-6.1)}} & \makecell{88.3\\ \textcolor{myred}{(+1.3)}} & \makecell{68.5\\ \textcolor{myred}{(+0.5)}} & \makecell{90.2\\ \textcolor{mygreen}{(-1.0)}} & \makecell{57.1\\ \textcolor{mygreen}{(-2.9)}} & \cellcolor{mygray} \makecell{77.3\\ \textcolor{mygreen}{(-1.1)}} \\
    + Rewarded & \textbf{96.3} & \underline{72.9} & \underline{91.3} & \underline{76.2} & \textbf{92.8} & \underline{72.9} & \cellcolor{mygray} \underline{83.7} \\
    + Unrewarded $\rightarrow$ Rewarded & \textbf{96.3} & \makecell{\textbf{79.5}\\ \textcolor{myred}{(+6.6)}} & \makecell{\textbf{91.7}\\ \textcolor{myred}{(+0.4)}} & \makecell{\textbf{78.2}\\ \textcolor{myred}{(+2.0)}} & \makecell{\underline{92.3}\\ \textcolor{mygreen}{(-0.5)}} & \makecell{\textbf{75.0}\\ \textcolor{myred}{(+2.1)}} & \cellcolor{mygray} \makecell{\textbf{85.5}\\ \textcolor{myred}{(+1.8)}} \\
    \bottomrule
    \end{tabular}
    \end{adjustbox}
    \label{tab:4}
\end{table}
\subsubsection{GUI agent tasks}\label{sec4.2.2} 
We used the dataset from~\citet{luo2025gui} for training Vision Language Models, which was collected and filtered from FindWeb~\citep{penedo2024fineweb}, UIBert~\citep{bai2021uibert}, AMEX~\citep{chai2024amex}, RICOSCA~\citep{li2020mapping}, Seeclick~\citep{cheng2024seeclick} and OS-Otlas~\citep{wu2024atlas}.
We evaluate the trained models on ScreenSpot~\citep{cheng2024seeclick}, ScreenSpot-V2~\citep{wu2024atlas}, ScreenSpot-Pro~\citep{li2025screenspot}, AndroidControll-Low~\citep{li2024effects}, AndroidControll-High, GUIAct-Web~\citep{li2025screenspot}, OmniAct-Web, OmniAct-Desktop~\citep{kapoor2024omniact} and GUI-Odyssey~\citep{lu2024gui} benchmarks, containing simple grounding tasks, low-level tasks and high-level tasks.

\subsection{Reward rules}\label{sec3.3}
\subsubsection{Mathematical reasoning task}\label{sec3.3.1}
The reward rule for training LLMs on mathematical reasoning tasks is solely based on accuracy, the reward is assigned only if the response answer is correct, our implementation is adopted from~\citet{zeng2025simplerl}.

\subsubsection{GUI agent task}\label{sec3.3.2}
For the GUI agent task, we adopt the reward rule from~\citet{luo2025gui}, which means that the reward is assigned by $r=\alpha r_{format}+\beta r_{acc}$, where $\alpha=0.2$ and $\beta=0.8$.
$r_{format}$ is assigned 1 if the response format follows the requirements (i.e., thinking process within \textless think\textgreater \textless /think\textgreater \text{ }tags and final answer within \textless answer\textgreater \textless /answer\textgreater \text{ }tags), and $r_{acc}=r_{act}+r_{point}+r_{text}$.
$r_{act}$ is assigned 1 if the predicted action is the same as the ground truth action.
$r_{point}$ is assigned 1 if the predicted point is in the ground truth bounding box (i.e., the predicted [x, y] in ground truth [x1, y1, x2, y2]).
$r_{text}$ is assigned 1 if the $F_1$ score between the predicted text and the ground truth text is larger than $0.5$.
\subsection{Hyper-parameters and Computational Environment}\label{sec4.3}
We provide detailed parameter settings here.
For the mathematical reasoning task, we trained models with the same parameters for both unrewarded exploration and rewarded training. 
Key hyper-parameters included $\text{steps}=500$, $\text{learning\_rate}=5\times10^{-7}$, KL coefficient $\beta=0.001$, $\text{batch\_size}=256$, $\text{max\_response\_length}=3072$, $\text{group\_size}=4$ and $\text{generation\_temperature}=1.0$.
For the GUI agent task, we also trained models with the same parameters for both unrewarded exploration and rewarded training. 
Key hyper-parameters included $\text{steps}=90$, $\text{learning\_rate}=1\times10^{-6}$, KL coefficient $\beta=0.01$, $\text{batch\_size}=128$, $\text{max\_response\_length}=1024$, $\text{group\_size}=5$, $\text{generation\_temperature}=1.0$.
All our experiments are conducted on NVIDIA-A800-80G GPUs.

\subsection{Results}\label{sec4.4}
\subsubsection{Mathematical reasoning task}\label{sec4.4.1}
From the experimental results (Table~\ref{tab:1} shows the best average results on the six mathematical benchmarks, and Table~\ref{tab:2} shows the best results on the six benchmarks, respectively), it can be seen that in the mathematical reasoning tasks, the performance of LLMs on the evaluation benchmark is improved to some extent even after LLMs undergo unrewarded exploration, which corresponds to the learning of rats in the absence of food rewards in Tolman's experiments~\citep{tolman1948cognitive}.
Further, by introducing rewards after performing unrewarded exploration, LLMs are able to achieve even slightly better performance after training than if they were trained directly with rewards throughout, which is similar to the learning results of Tolman's experiments in which rats were introduced to food rewards after exploration in the absence of food.
\begin{table}[!t]
    \setlength{\tabcolsep}{3pt}
    \caption{Model performance on low-level GUI tasks.
    Rewarded: training with rewards.
    Unrewarded: exploration without rewards.
    Type: the exact match score between the predicted action types and the ground truth.
    GR: grounding performance.
    SR: step-wise success rate.
    We show the relative difference at the \textit{+ Unrewarded} row compared to the base model and the \textit{+ Unrewarded $\rightarrow$ Rewarded} row compared to the \textit{+ Rewarded} row.}
    \begin{adjustbox}{width=0.5\textwidth, center}
    \begin{tabular}{lccccccccccccc}
    \toprule
    \multirow{2}{*}{\textbf{Model}} & 
    \multicolumn{3}{c}{\textbf{GUIAct-Web}} & 
    \multicolumn{3}{c}{\textbf{OmniAct-Web}} & 
    \multicolumn{3}{c}{\textbf{OmniAct-Desktop}} & 
    \multicolumn{3}{c}{\textbf{AndroidControl-Low}}\\
    \cmidrule(lr){2-4} \cmidrule(lr){5-7} \cmidrule(lr){8-10} \cmidrule(lr){11-13} & Type & GR & SR & Type & GR & SR & Type & GR & SR & Type & GR & SR & \\
    \midrule
    \multicolumn{14}{c}{\textit{Qwen 3B Models}} \\
    \midrule
    Qwen2.5-VL-3B-Instruct & 56.2 & 65.5 & 56.7 & 51.8 & 49.0 & 48.9 & 55.8 & 49.8 & 49.8 & 61.2 & 74.4 & 59.6 \\
    + Unrewarded & \makecell{\underline{88.2}\\ \textcolor{myred}{(+32.0)}} & \makecell{79.4\\ \textcolor{myred}{(+13.9)}} & \makecell{\underline{81.3}\\ \textcolor{myred}{(+24.6)}} & \makecell{77.4\\ \textcolor{myred}{(+25.6)}} & \makecell{67.0\\ \textcolor{myred}{(+28.0)}} & \makecell{67.0\\ \textcolor{myred}{(+28.0)}} & \makecell{\underline{88.0}\\ \textcolor{myred}{(+32.2)}} & \makecell{66.4\\ \textcolor{myred}{(+16.6)}} & \makecell{66.4\\ \textcolor{myred}{(+16.6)}} & \makecell{81.2\\ \textcolor{myred}{(+20.0)}} & \makecell{\textbf{82.8}\\ \textcolor{myred}{(+8.4)}} & \makecell{\textbf{67.4}\\ \textcolor{myred}{(+7.8)}} \\
    + Rewarded & \textbf{89.5} & \underline{87.2} & 73.9 & \textbf{87.8} & \textbf{74.4} & \textbf{74.4} & \textbf{91.9} & \underline{74.6} & \underline{74.6} & \underline{82.5} & 81.9 & 62.5 \\
    + Unrewarded $\rightarrow$ Rewarded & \makecell{\underline{88.2}\\ \textcolor{mygreen}{(-1.3)}} & \makecell{\textbf{87.6}\\ \textcolor{myred}{(+0.4)}} & \makecell{\textbf{87.1}\\ \textcolor{myred}{(+13.2)}} & \makecell{\underline{80.6}\\ \textcolor{mygreen}{(-7.2)}} & \makecell{\underline{73.7}\\ \textcolor{mygreen}{(-0.7)}} & \makecell{\underline{73.4}\\ \textcolor{mygreen}{(-1.0)}} & \makecell{87.6 \\ \textcolor{mygreen}{(-4.3)}} & \makecell{\textbf{78.8}\\ \textcolor{myred}{(+4.2)}} & \makecell{\textbf{78.8}\\ \textcolor{myred}{(+4.2)}} & \makecell{\textbf{85.5}\\ \textcolor{myred}{(+3.0)}} & \makecell{\underline{82.6}\\ \textcolor{myred}{(+0.7)}} & \makecell{\underline{66.9}\\ \textcolor{myred}{(+4.4)}}\\
    \midrule
    \multicolumn{14}{c}{\textit{Qwen 7B Models}} \\
    \midrule
    Qwen2.5-VL-7B-Instruct & 86.5 & 84.4 & \textbf{84.2} & 80.5 & 70.7 & 70.7 & 81.3 & 78.8 & 78.8 & 83.1 & \textbf{87.3} & \underline{72.5} \\
    + Unrewarded & \makecell{86.0\\ \textcolor{mygreen}{(-0.5)}} & \makecell{82.4\\ \textcolor{mygreen}{(-2.0)}} & \makecell{82.3\\ \textcolor{mygreen}{(-1.9)}} & \makecell{81.4\\ \textcolor{myred}{(+0.9)}} & \makecell{69.6\\ \textcolor{mygreen}{(-1.1)}} & \makecell{69.6\\ \textcolor{mygreen}{(-1.1)}} & \makecell{\underline{89.2}\\ \textcolor{myred}{(+7.9)}} & \makecell{82.4\\ \textcolor{myred}{(+3.6)}} & \makecell{82.4\\ \textcolor{myred}{(+3.6)}} & \makecell{78.6\\ \textcolor{mygreen}{(-4.5)}} & \makecell{\underline{87.0}\\ \textcolor{mygreen}{(-0.3)}} & \makecell{\textbf{72.7}\\ \textcolor{myred}{(+0.2)}} \\
    + Rewarded & \textbf{90.8} & \textbf{88.3} & 76.1 & \textbf{91.3} & \underline{77.0} & \underline{76.6} & \textbf{92.4} & \underline{82.5} & \underline{82.5} & \textbf{84.9} & 84.0 & 64.0 \\
    + Unrewarded $\rightarrow$ Rewarded & \textbf{90.8} & \makecell{\underline{86.7}\\ \textcolor{mygreen}{(-1.6)}} & \makecell{\underline{83.0}\\ \textcolor{myred}{(+6.9)}} & \makecell{\underline{85.3}\\ \textcolor{mygreen}{(-6.0)}} & \makecell{\textbf{77.2}\\ \textcolor{myred}{(+0.2)}} & \makecell{\textbf{76.8}\\ \textcolor{myred}{(+0.2)}} & \textbf{92.4} & \makecell{\textbf{83.6}\\ \textcolor{myred}{(+1.1)}} & \makecell{\textbf{83.6}\\ \textcolor{myred}{(+1.1)}} & \makecell{\underline{84.1}\\ \textcolor{mygreen}{(-0.8)}} & \makecell{82.3 \\ \textcolor{mygreen}{(-1.7)}} & \makecell{67.5 \\ \textcolor{myred}{(+3.5)}}\\
    \bottomrule
    \end{tabular}
    \end{adjustbox}
    \label{tab:5}
\end{table}
\subsubsection{GUI agent task}\label{sec4.4.2}
From the experimental results (Table~\ref{tab:3} and Table~\ref{tab:4} for GUI grounding tasks, Table~\ref{tab:5} for GUI low-level tasks and Table~\ref{tab:6} for GUI high-level tasks), we can observe that for GUI agent tasks requiring interaction with the environment and decision-making, LLMs demonstrate improved performance on benchmarks during unrewarded exploration phase.
This improvement is more pronounced than in mathematical reasoning tasks.
Furthermore, experimental results reveal that smaller model (i.e., Qwen2.5-VL-3B-Instruct) exhibits superior learning outcomes in unrewarded exploration compared to larger model (i.e., Qwen2.5-VL-7B-Instruct).
We have attempted to train LLaMA series vision language models (e.g. Llama-3.2-11B-Vision-Instruct); however, their built-in safety mechanisms caused them to systematically reject GUI agent instructions, preventing successful training, and therefore experiments on these models are not reported here.
The learning performance of LLMs on GUI tasks during the unrewarded exploration phase mirrors Tolman's rats experiments~\citep{tolman1948cognitive}: learning occurs even without rewards.
Upon introducing rewards, LLMs' performance continues to improve under rewards guidance, even slightly outperforming models trained exclusively with rewards, which also consistent with Tolman's rats studies.

Interestingly, we find that Qwen2.5-VL-3B-Instruct exhibits much larger gains during the unrewarded exploration stage than Qwen2.5-VL-7B-Instruct (e.g., +29.0 vs +2.0 in Table~\ref{tab:6}).
This pattern is consistent with observations in reinforcement learning at scale: weak models tend to be more sensitive to optimization signals, because their initial policies are less stable and more easily altered by exploratory updates.
In contrast, larger models begin with stronger priors and more coherent policies, which makes unguided exploration induce comparatively smaller behavioral changes.
In contrast to the 3B model, Qwen2.5-VL-7B-Instruct performs worse after both unrewarded exploration and reinforcement learning with rewards throughout than the base model on some benchmarks (e.g., AndroidControl-High in Table~\ref{tab:6}).
A plausible explanation is that the Qwen2.5-VL-7B-Instruct model already starts with a relatively strong, well-calibrated policy on these benchmarks, so both the unrewarded exploration and rewarded training can act as noisy or misaligned optimization signals.
\begin{table}[!t]
    \centering
    \setlength{\tabcolsep}{3pt}
    \caption{Model performance on high-level GUI tasks.
    Rewarded: training with rewards.
    Unrewarded: exploration without rewards.
    Type: the exact match score between the predicted action types and the ground truth.
    GR: grounding performance.
    SR: step-wise success rate.
    We show the relative difference at the \textit{+ Unrewarded} row compared to the base model and the \textit{+ Unrewarded $\rightarrow$ Rewarded} row compared to the \textit{+ Rewarded} row.}
    \begin{adjustbox}{width=0.5\textwidth, center} 
    \begin{tabular}{lccccccc}
    \toprule
    \multirow{2}{*}{\textbf{Model}} & 
    \multicolumn{3}{c}{\textbf{AndroidControl-High}} & 
    \multicolumn{3}{c}{\textbf{GUI-Odyssey}} \\
    \cmidrule(lr){2-4} \cmidrule(lr){5-7} & Type & GR & SR & Type & GR & SR & \\
    \midrule
    \multicolumn{8}{c}{\textit{Qwen 3B Models}} \\
    \midrule
    Qwen2.5-VL-3B-Instruct & 48.8 & 45.9 & 38.6 & 38.3 & 27.2 & 27.3 \\
    + Unrewarded & \makecell{\underline{64.1}\\ \textcolor{myred}{(+15.3)}} & \makecell{52.7\\ \textcolor{myred}{(+6.8)}} & \makecell{\underline{49.1}\\ \textcolor{myred}{(+10.5)}} & \makecell{\underline{67.3}\\ \textcolor{myred}{(+29.0)}} & \makecell{34.9\\ \textcolor{myred}{(+7.7)}} & \makecell{31.4\\ \textcolor{myred}{(+4.1)}} \\
    + Rewarded & 63.6 & \underline{59.4} & 45.4 & 54.7 & \underline{41.0} & \underline{34.8} \\
    + Unrewarded $\rightarrow$ Rewarded & \makecell{\textbf{69.5}\\ \textcolor{myred}{(+5.1)}} & \makecell{\textbf{60.0}\\ \textcolor{myred}{(+0.6)}} & \makecell{\textbf{50.6}\\ \textcolor{myred}{(+6.2)}} & \makecell{\textbf{68.5}\\ \textcolor{myred}{(+13.8)}} & \makecell{\textbf{45.4}\\ \textcolor{myred}{(+4.4)}} & \makecell{\textbf{41.4}\\ \textcolor{myred}{(+6.6)}}\\
    \midrule
    \multicolumn{8}{c}{\textit{Qwen 7B Models}} \\
    \midrule
    Qwen2.5-VL-7B-Instruct &  68.5 & 62.2 & \textbf{56.9} & 55.7 & 37.8 & 34.5 \\
    + Unrewarded & \makecell{67.1\\ \textcolor{mygreen}{(-1.4)}} & \makecell{59.3\\ \textcolor{mygreen}{(-2.9)}} & \makecell{\underline{56.2}\\ \textcolor{mygreen}{(-0.7)}} & \makecell{57.7\\ \textcolor{myred}{(+2.0)}} & \makecell{37.3\\ \textcolor{mygreen}{(-0.5)}} & \makecell{33.4\\ \textcolor{mygreen}{(-1.1)}} \\
    + Rewarded & \underline{71.8} & \textbf{65.5} & 50.9 & \underline{65.0} & \underline{43.3} & \underline{38.3} \\
    + Unrewarded $\rightarrow$ Rewarded & \makecell{\textbf{72.4}\\ \textcolor{myred}{(+1.4)}} & \makecell{\underline{62.7}\\ \textcolor{mygreen}{(-2.8)}} & \makecell{51.0 \\ \textcolor{myred}{(+0.1)}} & \makecell{\textbf{68.1}\\ \textcolor{myred}{(+3.1)}} & \makecell{\textbf{45.9}\\ \textcolor{myred}{(+2.6)}} & \makecell{\textbf{41.4}\\ \textcolor{myred}{(+3.1)}} \\
    \bottomrule
    \end{tabular}
    \end{adjustbox}
    \label{tab:6}
\end{table}
\section{Discussion}\label{sec5}
Our study began by investigating whether Tolman's theory of latent learning in rats also hold true for LLMs.
To investigate this question, we conducted experiments on mathematical reasoning tasks and GUI agent tasks using mainstream LLMs (such as Qwen and the Llama series).
Across both types of tasks, the results indicate that LLMs can improve their performance (i.e., acquire certain knowledge) even without rewards.
Furthermore, introducing rewards after unrewarded exploration yields slightly better model improvements than training with rewards throughout the entire process.
Taken together, our findings confirm that the latent learning phenomenon observed by Blodgett and Tolman~\citep{tolman1948cognitive, blodgett1929effect} in rats studies also manifests in LLMs.
Beyond these empirical observations, we also provide theoretical analyses for our experiments focusing on the unrewarded exploration phase, which prove that LLMs could learn without rewards.
This positions our work as an experimental step toward unifying cognitive psychology and machine intelligence under a shared framework of learning without reinforcement.

A primary limitation of our study is that both the empirical experiments and theoretical analysis are grounded in the GRPO~\citep{shao2024deepseekmath} algorithm.
While this choice provides a controlled and interpretable setting for examining latent learning dynamics, it remains unclear whether the same effects would persist under different RL algorithms.
Future work could extend our findings to other policy optimization methods, such as PPO~\citep{schulman2017proximal} or its variants, to test whether unrewarded exploration exhibits similar benefits across distinct training paradigms.
Such investigations would help determine whether latent learning in LLMs reflects a general property of reinforcement fine-tuning or is specific to the dynamics of GRPO.

Another limitation concerns the scope of models examined and the range of experimental tasks.
Our experiments focused on the Qwen and Llama series with parameter scales up to 8B.
While these models provide a representative testbed for controlled comparisons, they nonetheless may not fully capture the learning dynamics of other series models or larger models that possess more extensive pretraining corpora and emergent reasoning abilities.
Therefore, future work could extend these analyses to broader models series and larger models size to examine whether latent learning phenomena persist or intensify with increasing scale and diversity of model design.
We restricted our evaluation to mathematical reasoning and GUI agent environments, which represent structured yet limited domains of cognition.
Although these tasks are well suited for quantifying performance changes under rewarded and unrewarded conditions, they may not encompass the full spectrum of behaviors through which latent learning manifests.
Consequently, future studies could explore additional settings such as open-ended dialogue, embodied control, and code generation to test whether unrewarded exploration generalizes across broader cognitive domains.

Our work suggests that learning in LLMs may share deeper commonalities with biological cognition than previously recognized~\citep{lake2017building, Hassabis2017}.
Moreover, the emergence of latent learning in artificial systems indicates that the ability to acquire and organize knowledge without rewards may be a general property of intelligent agents~\citep{schmidhuber1991possibility, pathak2017curiosity, oudeyer2007intrinsic}, whether natural or artificial.
Taken together, this insight invites a reframing of how we conceptualize ``learning" in artificial intelligence, from a process driven purely by reinforcement or supervision to one that also involves unrewarded exploration and internal representation building.
Beyond theoretical implications, such parallels between animal and artificial learning could foster more integrated frameworks linking cognitive science, neuroscience, and machine learning, ultimately contributing to a unified understanding of intelligence~\citep{Botvinick2019, silver2021reward}.
Finally, by highlighting the existence of latent learning in LLMs, we believe that this work points towards artificial agents that can engage in more open-ended inquiry, a direction that may ultimately benefit scientific discovery, education, and human–AI collaboration, provided such systems are developed with appropriate safeguards.

\section{Conclusion}\label{sec6}
In this paper, we investigate whether Tolman's latent learning phenomenon exists in modern machine learning systems (i.e., LLMs). 
We conduct extensive experiments across mathematical reasoning and GUI agent tasks on different model series (i.e., Qwen and LLaMA), demonstrating that this phenomenon does indeed occur in LLMs: LLMs achieve measurable performance gains when learning without rewards, and continue to achieve higher performance gains when learning with rewards introduced later than when learning with rewards throughout the entire process. 
Beyond our empirical findings, we also provide theoretical analyses of our experimental setup to explain why LLMs achieve performance gains when learning without rewards.

By connecting a classic construct from cognitive psychology to the training dynamics of contemporary LLMs, our work suggests a potential research direction for future studies to investigate how such a two-stage training paradigm beginning with an unrewarded exploration phase may improve LLMs' performance, especially at much larger scales.

\bibliography{main}
\bibliographystyle{icml2026}

\newpage
\appendix
\onecolumn
\section{Proof of Theorems}\label{appendix:proof}
\textbf{Theorem 1} (Performance improvement without rewards in sparse setting)
\textit{Consider an autoregressive language model policy $\pi_\theta$ that defines, for each input prompt $x$ and token prefix $y_{<t}$, a conditional distribution\(\pi_\theta(y_t \mid x,y_{<t})\) over the next token $y_t$. The model is trained with Group Relative Policy Optimization where the per-token advantage is fixed to $A_{i,t}\equiv 1$, ratio clipping is applied with parameter $\varepsilon>0$, and a KL trust-region penalty of strength $\beta>0$ is included. Let the reference policy be $\pi_{\text{ref}}(y_t \mid x,y_{<t})$, and let the proposal distribution be $\pi_{\text{old}}(y_t \mid x,y_{<t})$. Assume that the likelihood ratio \(h(x,y_{<t},y_t) \;=\; \frac{\pi_{\text{old}}(y_t \mid x,y_{<t})}{\pi_{\text{ref}}(y_t \mid x,y_{<t})}\) is nondecreasing in the latent utility $u^\star(x,y_{<t},y_t)$ for every input $x$, every prefix $y_{<t}$, and every token $y_t$.}

\textit{Then the GRPO update produces an updated conditional policy of the closed form}
\begin{equation}
    \pi^\star_\theta(y_t \mid x,y_{<t}) \;=\; \min\!\Big((1+\varepsilon)\,\pi_{\text{old}}(y_t \mid x,y_{<t}),\ \ \tau(x,y_{<t})\,\pi_{\text{ref}}(y_t \mid x,y_{<t})\Big),
\end{equation}
\textit{where $\tau(x,y_{<t})>0$ is the unique normalizing constant satisfying \(
\sum_{y_t \in \mathcal V} \pi^\star_\theta(y_t \mid x,y_{<t}) = 1.
\) Moreover, under the monotone-likelihood-ratio condition, the expected performance}
\begin{equation}
    J(\theta) \;=\; \mathbb E_{x \sim \mathcal D}\,\mathbb E_{y \sim \pi_\theta(\cdot \mid x)} \Bigg[\sum_{t=1}^{|y|} u^\star(x,y_{<t},y_t)\Bigg]
\end{equation}

\textit{satisfies \(
J(\pi^\star_\theta) \;\ge\; J(\pi_{\text{ref}}).
\) Thus even with no scalar reward signal and with constant advantage, ratio-clipped GRPO with KL regularization yields monotone improvement in the expected utility.}

\begin{proof}
The model is a conditional distribution $\pi_\theta(y_t \mid x,y_{<t})$ over tokens in an output sequence $y=(y_1,\ldots,y_T)$ given an input $x$. The autoregressive factorization and log-likelihood decomposition are
\begin{equation}\label{eq.1}
\pi_\theta(y\mid x)=\prod_{t=1}^{T}\pi_\theta(y_t\mid x,y_{<t}),
\qquad 
\log \pi_\theta(y\mid x)=\sum_{t=1}^{T}\log \pi_\theta(y_t\mid x,y_{<t}).
\end{equation}
For each $x$ we sample a group $G=\{y^{(1)},\dots,y^{(K)}\}$ of complete outputs from a fixed proposal policy $\pi_{\mathrm{prop}}(y\mid x)$ (for example a frozen copy of $\pi_{\text{old}}(y\mid x)$). In the setting of interest, the advantage is set identically to one at every token. Keeping ratio clipping and a direct KL regularizer, the GRPO training objective (ignoring length-normalization constants that do not affect the derivations) is
\begin{equation}\label{eq.2}
\begin{split}
    \mathcal L_{\text{A=1}}(\theta)
&=\mathbb{E}_{x\sim \mathcal D,\ G\sim \pi_{\mathrm{prop}}^K}\Bigg[
\sum_{i=1}^{K}\sum_{t=1}^{|y^{(i)}|}
\min\!\Big(\rho_{i,t}(\theta),\ \operatorname{clip}\big(\rho_{i,t}(\theta);\,1-\varepsilon,\,1+\varepsilon\big)\Big)\Bigg]
\;\\&-\;\beta\,\mathbb{E}_{x,t}\Big[\mathrm{KL}\big(\pi_\theta(\cdot\mid x,y_{<t})\ \|\ \pi_{\text{ref}}(\cdot\mid x,y_{<t})\big)\Big],
\end{split}
\end{equation}
where the per-token importance ratio and its log-increment are
\begin{equation}
\begin{split}
&\rho_{i,t}(\theta)=
\frac{\pi_\theta(y^{(i)}_t\mid x,y^{(i)}_{<t})}{\pi_{\mathrm{prop}}(y^{(i)}_t\mid x,y^{(i)}_{<t})}
=\exp\!\Big\{\Delta_{i,t}(\theta)\Big\},\\&
\qquad 
\Delta_{i,t}(\theta)=\log\pi_\theta(y^{(i)}_t\mid x,y^{(i)}_{<t})-\log\pi_{\mathrm{prop}}(y^{(i)}_t\mid x,y^{(i)}_{<t}).
\end{split}
\end{equation}
The clipping operator is the standard two-sided clamp
\begin{equation}
\operatorname{clip}(\rho;1-\varepsilon,1+\varepsilon)=
\min\big\{\max(\rho,1-\varepsilon),\,1+\varepsilon\big\}.
\end{equation}
With $A\equiv 1$, the inner minimum becomes $\min(\rho_{i,t}(\theta),\,1+\varepsilon)$ at every token. This enables a per-state rewriting. For any fixed token state $(x,y_{<t})$ with discrete vocabulary $\mathcal V$ and any constant $c>0$,
{
\footnotesize
\begin{equation}
    \sum_{y_t\in\mathcal V} \pi_{\mathrm{prop}}(y_t\mid x,y_{<t}) \,\min\!\Big(\frac{\pi_\theta(y_t\mid x,y_{<t})}{\pi_{\mathrm{prop}}(y_t\mid x,y_{<t})},\,c\Big)
=\sum_{y_t\in\mathcal V}\min\!\Big(\pi_\theta(y_t\mid x,y_{<t}),\,c\,\pi_{\mathrm{prop}}(y_t\mid x,y_{<t})\Big).
\end{equation}}
Setting $c=1+\varepsilon$ and writing the KL penalty at the same state yields, for each $(x,y_{<t})$, the concave per-state objective
\begin{equation}
\begin{split}
    \ell_{x,y_{<t}}(\pi)
&=\sum_{y_t\in\mathcal V} \min\!\big(\pi(y_t\mid x,y_{<t}),\,(1+\varepsilon)\,\pi_{\mathrm{prop}}(y_t\mid x,y_{<t})\big)\ \\&
-\ \beta\sum_{y_t\in\mathcal V}\pi(y_t\mid x,y_{<t})\log\frac{\pi(y_t\mid x,y_{<t})}{\pi_{\text{ref}}(y_t\mid x,y_{<t})},
\end{split}
\end{equation}
to be maximized over the simplex $\{\pi(\cdot\mid x,y_{<t})\in\Delta^{|\mathcal V|-1}\}$. The map $\pi\mapsto \sum_{y_t}\min(\pi,\text{const})$ is concave coordinatewise and $-\beta\,\mathrm{KL}(\pi\|\pi_{\text{ref}})$ is concave in its first argument, so $\ell_{x,y_{<t}}$ is concave.

Introduce a Lagrange multiplier $\lambda_{x,y_{<t}}$ for the constraint $\sum_{y_t}\pi(y_t\mid x,y_{<t})=1$ and nonnegativity multipliers $\mu_{x,y_{<t}}(y_t)\ge 0$. The subgradient of the capped term is
{
\tiny
\begin{equation}
\partial_{\pi(y_t\mid x,y_{<t})}\,\min\!\big(\pi(y_t\mid x,y_{<t}),(1+\varepsilon)\pi_{\mathrm{prop}}(y_t\mid x,y_{<t})\big)=
\begin{cases}
\{1\},& \pi(y_t\mid x,y_{<t})<(1+\varepsilon)\pi_{\mathrm{prop}}(y_t\mid x,y_{<t}),\\
[0,1],& \pi(y_t\mid x,y_{<t})=(1+\varepsilon)\pi_{\mathrm{prop}}(y_t\mid x,y_{<t}),\\
\{0\},& \pi(y_t\mid x,y_{<t})>(1+\varepsilon)\pi_{\mathrm{prop}}(y_t\mid x,y_{<t}).
\end{cases}
\end{equation}}
Hence the KKT stationarity condition states that there exists $\xi_{x,y_{<t}}(y_t)\in[0,1]$ with
\begin{equation}\label{eq.8}
\begin{split} 
\xi_{x,y_{<t}}(y_t)\cdot \mathbf 1\!\left\{\pi(y_t\mid x,y_{<t})\right. & \left.\le (1+\varepsilon)\pi_{\mathrm{prop}}(y_t\mid x,y_{<t})\right\}
\;+\;0\cdot \mathbf 1\!\left\{\pi(y_t\mid x,y_{<t})\right. \\&
\left.>(1+\varepsilon)\pi_{\mathrm{prop}}(y_t\mid x,y_{<t})\right\}\\&
\;-\;\beta\Big(\log \tfrac{\pi(y_t\mid x,y_{<t})}{\pi_{\text{ref}}(y_t\mid x,y_{<t})}+1\Big)\;-\;\lambda_{x,y_{<t}}\;+\;\mu_{x,y_{<t}}(y_t)=0,
\end{split}
\end{equation}
with complementary slackness $\mu_{x,y_{<t}}(y_t)\,\pi(y_t\mid x,y_{<t})=0$. Whenever $\pi(y_t\mid x,y_{<t})>0$ we have $\mu_{x,y_{<t}}(y_t)=0$. In the interior region where the cap is inactive,
\begin{equation}
    \pi(y_t\mid x,y_{<t})<(1+\varepsilon)\pi_{\mathrm{prop}}(y_t\mid x,y_{<t}),
\end{equation}
the subgradient equals \eqref{eq.1} and \eqref{eq.8} gives
\begin{equation}\label{eq.10}
1-\beta\Big(\log \tfrac{\pi(y_t\mid x,y_{<t})}{\pi_{\text{ref}}(y_t\mid x,y_{<t})}+1\Big)-\lambda_{x,y_{<t}}=0
\quad\Rightarrow\quad
\pi(y_t\mid x,y_{<t})=\tau_{x,y_{<t}}\,\pi_{\text{ref}}(y_t\mid x,y_{<t}),
\end{equation}
where $\tau_{x,y_{<t}}\equiv \exp\!\big(\tfrac{1-\lambda_{x,y_{<t}}}{\beta}-1\big)$. Combining the interior formula with the active-cap condition $\pi(y_t\mid x,y_{<t})=(1+\varepsilon)\pi_{\mathrm{prop}}(y_t\mid x,y_{<t})$ yields the single closed-form expression for the maximizer at state $(x,y_{<t})$:
{
\footnotesize
\begin{equation}
\quad
\pi^\star_\theta(y_t\mid x,y_{<t})\;=\;\min\Big((1+\varepsilon)\,\pi_{\mathrm{prop}}(y_t\mid x,y_{<t}),\ \ \tau_{x,y_{<t}}\,\pi_{\text{ref}}(y_t\mid x,y_{<t})\Big), \sum_{y_t\in\mathcal V}\pi^\star_\theta(y_t\mid x,y_{<t})=1\ .
\quad
\end{equation}}
Strict concavity of $\ell_{x,y_{<t}}$ in the density ratios $\log\frac{\pi}{\pi_{\text{ref}}}$ (via the KL term) ensures the uniqueness of $\tau_{x,y_{<t}}>0$.

It is convenient to partition the vocabulary at state $(x,y_{<t})$ according to whether the cap is active:
\begin{equation}
S_{x,y_{<t}}=\left\{y_t:\ (1+\varepsilon)\,\pi_{\mathrm{prop}}(y_t\mid x,y_{<t})\ \le\ \tau_{x,y_{<t}}\,\pi_{\text{ref}}(y_t\mid x,y_{<t})\right\},
\qquad 
T_{x,y_{<t}}=\mathcal V\setminus S_{x,y_{<t}}.
\end{equation}
By construction,
\begin{equation}\label{eq.11}
\pi^\star_\theta(y_t\mid x,y_{<t})=
\begin{cases}
(1+\varepsilon)\,\pi_{\mathrm{prop}}(y_t\mid x,y_{<t}),& y_t\in S_{x,y_{<t}},\\
\tau_{x,y_{<t}}\,\pi_{\text{ref}}(y_t\mid x,y_{<t}),& y_t\in T_{x,y_{<t}},
\end{cases}
\end{equation}
and summing the two pieces gives the exact mass-balance identity
\begin{equation}\label{eq.12}
\sum_{y_t\in S_{x,y_{<t}}}\Big((1+\varepsilon)\,\pi_{\mathrm{prop}}(y_t\mid x,y_{<t})-\pi_{\text{ref}}(y_t\mid x,y_{<t})\Big)
\;+\;
\sum_{y_t\in T_{x,y_{<t}}}\Big(\tau_{x,y_{<t}}-1\Big)\pi_{\text{ref}}(y_t\mid x,y_{<t})\;=\;0.
\end{equation}

We now relate the update \eqref{eq.10} to the latent utility. Let $u^\star(x,y_{<t},y_t)$ be the (unobserved) tokenwise utility. The expected performance is
\begin{equation}
J(\theta)=\mathbb E_{x\sim\mathcal D}\,\mathbb E_{y\sim \pi_\theta(\cdot\mid x)}\Big[\sum_{t=1}^{|y|}u^\star(x,y_{<t},y_t)\Big]
=\mathbb E_{x,y_{<t}}\sum_{y_t\in\mathcal V}\pi_\theta(y_t\mid x,y_{<t})\,u^\star(x,y_{<t},y_t).
\end{equation}
Assume the monotone-likelihood-ratio condition: for every $(x,y_{<t})$ and any tokens $y_t^{(1)},y_t^{(2)}$,
\begin{equation}
u^\star(x,y_{<t},y_t^{(1)})\ \ge\ u^\star(x,y_{<t},y_t^{(2)})\quad\Rightarrow\quad
\frac{\pi_{\mathrm{prop}}(y_t^{(1)}\mid x,y_{<t})}{\pi_{\text{ref}}(y_t^{(1)}\mid x,y_{<t})}
\ \ge\ 
\frac{\pi_{\mathrm{prop}}(y_t^{(2)}\mid x,y_{<t})}{\pi_{\text{ref}}(y_t^{(2)}\mid x,y_{<t})}.
\end{equation}
Define the likelihood ratio
\begin{equation}
    h(x,y_{<t},y_t)=\frac{\pi_{\mathrm{prop}}(y_t\mid x,y_{<t})}{\pi_{\text{ref}}(y_t\mid x,y_{<t})}
\end{equation}
and the tilt
\begin{equation}
    w_{x,y_{<t}}(y_t)=\frac{\pi^\star_\theta(y_t\mid x,y_{<t})}{\pi_{\text{ref}}(y_t\mid x,y_{<t})}
=\min\Big(\tau_{x,y_{<t}},\ (1+\varepsilon)\,h(x,y_{<t},y_t)\Big).
\end{equation}
Both $w_{x,y_{<t}}(y_t)$ and $u^\star(x,y_{<t},y_t)$ are nondecreasing functions of the common scalar $h(x,y_{<t},y_t)$. Chebyshev’s association inequality under the base measure $\pi_{\text{old}}(\cdot\mid x,y_{<t})$ therefore gives
\begin{equation}\label{eq.14}
\begin{split}
     \sum_{y_t\in\mathcal V}\pi^\star_\theta(y_t\mid x,y_{<t})\,u^\star(x,y_{<t},y_t)
= \sum_{y_t\in\mathcal V}\pi_{\text{old}}(y_t\mid x,y_{<t})\,w_{x,y_{<t}}(y_t)\,u^\star(x,y_{<t},y_t)
\ \\\ge\ \Big(\sum_{y_t}\pi_{\text{old}}\,w_{x,y_{<t}}\Big)\Big(\sum_{y_t}\pi_{\text{old}}\,u^\star\Big)
= \sum_{y_t\in\mathcal V}\pi_{\text{old}}(y_t\mid x,y_{<t})\,u^\star(x,y_{<t},y_t),
\end{split}
\end{equation}
because $\sum_{y_t}\pi_{\text{old}}(y_t\mid x,y_{<t})\,w_{x,y_{<t}}(y_t)=\sum_{y_t}\pi^\star_\theta(y_t\mid x,y_{<t})=1$. Averaging over $(x,y_{<t})$ yields the global monotone improvement
\begin{equation}
    J(\pi^\star_\theta)\ \ge\ J(\pi_{\text{old}}).
\end{equation}

The same conclusion can be expressed in the partition language. Using \eqref{eq.11}-\eqref{eq.12}, the per-state change in performance is
\begin{equation}\label{eq.15}
\begin{split}
    \Delta J_{x,y_{<t}}\ =\ \sum_{y_t\in \mathcal V}\Big(\pi^\star_\theta(y_t\mid x,y_{<t})-\pi_{\text{old}}(y_t\mid x,y_{<t})\Big)\,u^\star(x,y_{<t},y_t)
\\=\sum_{y_t\in S_{x,y_{<t}}}\Big((1+\varepsilon)\pi_{\mathrm{prop}}-\pi_{\text{ref}}\Big)u^\star
+\sum_{y_t\in T_{x,y_{<t}}}\Big(\tau_{x,y_{<t}}-1\Big)\pi_{\text{ref}}\,u^\star,
\end{split}
\end{equation}
where we have abbreviated the arguments to reduce clutter in the last line. Define the nonnegative transfer magnitude
\begin{equation}\label{eq.16}
\begin{split}
&M_{x,y_{<t}}\ \equiv\ \sum_{y_t\in T_{x,y_{<t}}}\Big(\tau_{x,y_{<t}}-1\Big)\pi_{\text{ref}}(y_t\mid x,y_{<t})
\ \\&=\ -\sum_{y_t\in S_{x,y_{<t}}}\Big((1+\varepsilon)\pi_{\mathrm{prop}}(y_t\mid x,y_{<t})-\pi_{\text{ref}}(y_t\mid x,y_{<t})\Big)\ \ \ge\ 0.
\end{split}
\end{equation}
Introduce the increment-weighted conditional means
\begin{equation}
    \bar u^{(+)}_{S_{x,y_{<t}}}\ =\ \frac{\sum_{y_t\in S_{x,y_{<t}}}\Big(\pi_{\text{ref}}(y_t\mid x,y_{<t})-(1+\varepsilon)\pi_{\mathrm{prop}}(y_t\mid x,y_{<t})\Big)\,u^\star(x,y_{<t},y_t)}{\sum_{y_t\in S_{x,y_{<t}}}\Big(\pi_{\text{ref}}(y_t\mid x,y_{<t})-(1+\varepsilon)\pi_{\mathrm{prop}}(y_t\mid x,y_{<t})\Big)},
\end{equation}
\begin{equation}
    \bar u^{(-)}_{T_{x,y_{<t}}}\ =\ \frac{\sum_{y_t\in T_{x,y_{<t}}}\Big(\tau_{x,y_{<t}}-1\Big)\pi_{\text{ref}}(y_t\mid x,y_{<t})\,u^\star(x,y_{<t},y_t)}{\sum_{y_t\in T_{x,y_{<t}}}\Big(\tau_{x,y_{<t}}-1\Big)\pi_{\text{ref}}(y_t\mid x,y_{<t})}.
\end{equation}

Then \eqref{eq.15} can be rewritten as
\begin{equation}\label{eq.17}
\Delta J_{x,y_{<t}}\ =\ -\,M_{x,y_{<t}}\big(\bar u^{(+)}_{S_{x,y_{<t}}}-\bar u^{(-)}_{T_{x,y_{<t}}}\big).
\end{equation}
On $S_{x,y_{<t}}$ the defining inequality \((1+\varepsilon)\pi_{\mathrm{prop}}\le \tau_{x,y_{<t}}\pi_{\text{ref}}\) is equivalent to \(h(x,y_{<t},y_t)\le \tau_{x,y_{<t}}/(1+\varepsilon)\), while the reverse holds on $T_{x,y_{<t}}$. Under the monotone-likelihood-ratio condition \eqref{eq.14}, the conditional mean of $u^\star$ on $T_{x,y_{<t}}$ is at least as large as on $S_{x,y_{<t}}$, which implies $\bar u^{(-)}_{T_{x,y_{<t}}}\ge \bar u^{(+)}_{S_{x,y_{<t}}}$. Plugging this into \eqref{eq.17} gives $\Delta J_{x,y_{<t}}\ge 0$. Averaging once more over $(x,y_{<t})$ recovers $J(\pi^\star_\theta)\ge J(\pi_{\text{ref}})$.

Finally, the tokenwise KL penalty admits an unbiased single-sample estimator using $\psi(r)=r-\log r-1$. For any $(x,y_{<t})$,
\begin{equation}\label{eq.18}
\mathbb E_{y_t\sim \pi_\theta(\cdot\mid x,y_{<t})}\left[\psi\!\Big(\frac{\pi_{\text{ref}}(y_t\mid x,y_{<t})}{\pi_\theta(y_t\mid x,y_{<t})}\Big)\right]
=\mathrm{KL}\big(\pi_\theta(\cdot\mid x,y_{<t})\ \|\ \pi_{\text{ref}}(\cdot\mid x,y_{<t})\big),
\end{equation}
and $\psi(r)\ge 0$ for all $r>0$ since $\psi''(r)=1/r>0$ and $\psi(1)=0$. This validates the KL term used in \eqref{eq.2}. In small-step stochastic ascent, if $\theta^+=\theta+\eta \hat g$ with $\mathbb E[\hat g]=\nabla_\theta \mathcal L_{\text{A=1}}(\theta)$ and bounded second moment, standard smoothness arguments imply an expected increase in the surrogate. Translating to performance via the association argument above yields the first-order estimate
{
\footnotesize
\begin{equation}\label{eq.19}
\mathbb{E}\big[J(\theta^{+})-J(\theta)\big]\ \approx\ \eta\ \mathbb{E}_{x,y_{<t}}\Big[\,\mathrm{Cov}_{y_t\sim \pi_{\text{ref}}(\cdot\mid x,y_{<t})}\big(\min(\tau_{x,y_{<t}},(1+\varepsilon)h(x,y_{<t},y_t)),\,u^\star(x,y_{<t},y_t)\big)\,\Big]\ \ge\ 0,
\end{equation}
}
because both arguments of the covariance are nondecreasing in $h(x,y_{<t},y_t)$.

In summary, with advantage fixed to one and no scalar rewards, the GRPO objective with ratio clipping and a KL trust region reduces at each token state $(x,y_{<t})$ to maximizing the overlap $\sum_{y_t}\min(\pi(y_t\mid x,y_{<t}),(1+\varepsilon)\pi_{\mathrm{prop}}(y_t\mid x,y_{<t}))$ minus a KL term to $\pi_{\text{ref}}(\cdot\mid x,y_{<t})$. The unique maximizer is the water-filling solution \eqref{eq.10}. Under the monotone-likelihood-ratio condition \eqref{eq.14}, this update increases the expected latent utility, hence $J(\pi^\star_\theta)\ge J(\pi_{\text{ref}})$, completing the proof.
\end{proof}

\textbf{Theorem 2} (Performance improvement without rewards in continuous setting)
\textit{Fix an input distribution $\mathcal D$, an autoregressive language model policy $\pi_\theta(y_t\mid x,y_{<t})$ that admits a density with respect to a common $\sigma$-finite base measure on the token space, a reference policy $\pi_{\text{ref}}(y_t\mid x,y_{<t})$ with the same property, and a proposal (teacher) policy $\pi_{\mathrm{prop}}(y_t\mid x,y_{<t})$. Train with Group Relative Policy Optimization in which the per-token advantage is fixed to $A\equiv 1$, the ratio is clipped with parameter $\varepsilon>0$ so that the per-token factor becomes $\min(\rho,1+\varepsilon)$, and a KL trust-region penalty of strength $\beta>0$ is included. Assume that for every token state $(x,y_{<t})$ the likelihood ratio $h(x,y_{<t},y_t)=\pi_{\mathrm{prop}}(y_t\mid x,y_{<t})/\pi_{\text{ref}}(y_t\mid x,y_{<t})$ is comonotone with the latent token utility $u^\star(x,y_{<t},y_t)$ under the base distribution $\pi_{\text{ref}}(\cdot\mid x,y_{<t})$, meaning that $(u^\star(\cdot)-u^\star(\cdot'))(h(\cdot)-h(\cdot'))\ge 0$ almost everywhere. Then the statewise maximizer of the GRPO surrogate has the water-filling form}
\begin{gather*}
    \pi^\star_\theta(y_t\mid x,y_{<t})=\min\big((1+\varepsilon)\,\pi_{\mathrm{prop}}(y_t\mid x,y_{<t}),\ \tau_{x,y_{<t}}\,\pi_{\text{ref}}(y_t\mid x,y_{<t})\big),\\    
    \int \pi^\star_\theta(y_t\mid x,y_{<t})\,\mathrm d\nu(y_t)=1,
\end{gather*}

\textit{for a unique normalizing constant $\tau_{x,y_{<t}}>0$. Moreover, the expected performance}
\[
J(\theta)=\mathbb E_{x\sim\mathcal D}\,\mathbb E_{y\sim \pi_\theta(\cdot\mid x)}\Big[\sum_{t=1}^{|y|}u^\star(x,y_{<t},y_t)\Big]
=\mathbb E_{x,y_{<t}}\int \pi_\theta(y_t\mid x,y_{<t})\,u^\star(x,y_{<t},y_t)\,\mathrm d\nu(y_t)
\]
\textit{satisfies $J(\pi^\star_\theta)\ge J(\pi_{\text{ref}})$. Hence, even without scalar rewards and with constant advantage, ratio-clipped GRPO with KL regularization yields monotone improvement in expected latent utility in continuous output spaces.}
\begin{proof}
The proof proceeds in the continuous measure-theoretic setting by replacing sums with integrals and by making all densities explicit with respect to a common $\sigma$-finite base measure. Fix an input $x$ and a token state $(x,y_{<t})$. Let $(\mathcal Y,\mathcal F,\nu)$ be a measurable token space with a $\sigma$-finite base measure $\nu$ such that the conditional distributions $\pi_\theta(\cdot\mid x,y_{<t})$, $\pi_{\text{ref}}(\cdot\mid x,y_{<t})$, and $\pi_{\mathrm{prop}}(\cdot\mid x,y_{<t})$ admit densities with respect to $\nu$. The autoregressive factorization and the log-likelihood decomposition remain unchanged,
\begin{equation}\label{t.1}
\pi_\theta(y\mid x)=\prod_{t=1}^T \pi_\theta(y_t\mid x,y_{<t}),\qquad
\log\pi_\theta(y\mid x)=\sum_{t=1}^T \log \pi_\theta(y_t\mid x,y_{<t}). 
\end{equation}
At each $(x,y_{<t})$, complete sequences are proposed from $\pi_{\mathrm{prop}}(\cdot\mid x,y_{<t})$. Because the per-token advantage is set to $A\equiv 1$, the ratio-clipped token factor becomes $\min(\rho,1+\varepsilon)$ where
\begin{equation}\label{t.2}
\begin{split}
    \rho_{i,t}(\theta)=\frac{\pi_\theta(y^{(i)}_t\mid x,y^{(i)}_{<t})}{\pi_{\mathrm{prop}}(y^{(i)}_t\mid x,y^{(i)}_{<t})}
=\exp\{\Delta_{i,t}(\theta)\},\\
\Delta_{i,t}(\theta)=\log\pi_\theta(y^{(i)}_t\mid x,y^{(i)}_{<t})-\log\pi_{\mathrm{prop}}(y^{(i)}_t\mid x,y^{(i)}_{<t}), 
\end{split}
\end{equation}

and the two-sided clip is $\operatorname{clip}(\rho;1-\varepsilon,1+\varepsilon)=\min\{\max(\rho,1-\varepsilon),\,1+\varepsilon\}$. Since $\min(\rho,\operatorname{clip}(\rho))=\min(\rho,1+\varepsilon)$ holds identically, expectations over sampled tokens can be rewritten as integrals over densities. For any constant $c>0$ one has the identity
\begin{equation}\label{t.3}
\begin{split}
    \int_{\mathcal Y} \pi_{\mathrm{prop}}(y_t\mid x,y_{<t})\,\min\!\Big(\frac{\pi_\theta(y_t\mid x,y_{<t})}{\pi_{\mathrm{prop}}(y_t\mid x,y_{<t})},\,c\Big)\, \mathrm d\nu(y_t)
=\int_{\mathcal Y} \min\!\Big(\pi_\theta(y_t\mid x,y_{<t}),\,c\,\pi_{\mathrm{prop}}(y_t\mid x,y_{<t})\Big)\, \mathrm d\nu(y_t). 
\end{split}
\end{equation}
Setting $c=1+\varepsilon$ and writing the KL penalty at the same token state as a functional of the density with respect to $\nu$ gives, for each $(x,y_{<t})$, the continuous surrogate
\begin{equation}\label{t.4}
\scalebox{0.93}{$
\begin{aligned}
\ell_{x,y_{<t}}(\pi)=
\int_{\mathcal Y} \min\!\Big(\pi(y_t\mid x,y_{<t}),\,(1+\varepsilon)\,\pi_{\mathrm{prop}}(y_t\mid x,y_{<t})\Big)\,\mathrm d\nu(y_t)
- \beta \int_{\mathcal Y} \pi(y_t\mid x,y_{<t}) \log\frac{\pi(y_t\mid x,y_{<t})}{\pi_{\text{ref}}(y_t\mid x,y_{<t})}\, \mathrm d\nu(y_t),
\end{aligned}
$}
\end{equation}
to be maximized over the probability simplex of densities at $(x,y_{<t})$,
\begin{equation}\label{t.5}
\pi(\cdot\mid x,y_{<t})\ge 0\ \text{almost everywhere},\qquad 
\int_{\mathcal Y} \pi(y_t\mid x,y_{<t})\,\mathrm d\nu(y_t)=1.
\end{equation}
The map $a\mapsto \min(a,\alpha)$ is concave for every $\alpha>0$, integrals preserve concavity, and $-\beta\,\mathrm{KL}(\pi\|\pi_{\text{ref}})$ is strictly concave in its first argument. Therefore $\ell_{x,y_{<t}}$ is a strictly concave functional on the convex set \eqref{t.5}.

To derive the closed-form solution, introduce the Lagrangian
\begin{equation}\label{t.6}
\begin{split}
    \mathcal{L}(\pi,\lambda_{x,y_{<t}},\mu_{x,y_{<t}})=\ell_{x,y_{<t}}(\pi)
-\lambda_{x,y_{<t}}\!\left(\int_{\mathcal Y} \pi(y_t\mid x,y_{<t})\,\mathrm d\nu(y_t)-1\right)
-\int_{\mathcal Y}\mu_{x,y_{<t}}(y_t)\,\pi(y_t\mid x,y_{<t})\,\mathrm d\nu(y_t), 
\end{split}
\end{equation}
where $\lambda_{x,y_{<t}}\in\mathbb R$ is the multiplier for normalization and $\mu_{x,y_{<t}}(\cdot)\ge 0$ is the multiplier for nonnegativity. Using the pointwise subgradient of the capped term,
\begin{equation}\label{t.7}
\partial_{\pi}\,\min\!\big(\pi,\,(1+\varepsilon)\pi_{\mathrm{prop}}\big)=
\begin{cases}
\{1\}, & \pi<(1+\varepsilon)\pi_{\mathrm{prop}},\\
[0,1], & \pi=(1+\varepsilon)\pi_{\mathrm{prop}},\\
\{0\}, & \pi>(1+\varepsilon)\pi_{\mathrm{prop}},
\end{cases}
\end{equation}
and $\partial_{\pi}\big[-\beta\,\pi\log(\pi/\pi_{\text{ref}})\big]=-\beta(\log(\pi/\pi_{\text{ref}})+1)$, the KKT stationarity condition states that there exists a measurable $\xi_{x,y_{<t}}(y_t)\in[0,1]$ such that for almost every $y_t$,
\begin{equation}\label{t.8}
\scalebox{0.66}{$
\begin{aligned}
\xi_{x,y_{<t}}(y_t)\cdot \mathbf 1\!\Big\{\pi(y_t\mid x,y_{<t})\le (1+\varepsilon)\pi_{\mathrm{prop}}(y_t\mid x,y_{<t})\Big\}
+ 0\cdot \mathbf 1\!\Big\{\pi(y_t\mid x,y_{<t})>(1+\varepsilon)\pi_{\mathrm{prop}}(y_t\mid x,y_{<t})\Big\}
- \beta\Big(\log \tfrac{\pi(y_t\mid x,y_{<t})}{\pi_{\text{ref}}(y_t\mid x,y_{<t})}+1\Big)
- \lambda_{x,y_{<t}} + \mu_{x,y_{<t}}(y_t)=0, 
\end{aligned}
$}
\end{equation}
together with complementary slackness $\mu_{x,y_{<t}}(y_t)\,\pi(y_t\mid x,y_{<t})=0$. Whenever $\pi(y_t\mid x,y_{<t})>0$ we have $\mu_{x,y_{<t}}(y_t)=0$. On the interior region where the cap is inactive, namely $\pi<(1+\varepsilon)\pi_{\mathrm{prop}}$, the subgradient equals $1$ and equation~\eqref{eq8} reads
\begin{equation}\label{eq8}
    1-\beta\Big(\log \tfrac{\pi(y_t\mid x,y_{<t})}{\pi_{\text{ref}}(y_t\mid x,y_{<t})}+1\Big)-\lambda_{x,y_{<t}}=0,
\end{equation}
which is equivalent to the existence of a scalar
\begin{equation}
    \tau_{x,y_{<t}}=\exp\!\Big(\tfrac{1-\lambda_{x,y_{<t}}}{\beta}-1\Big)>0
\end{equation}
such that
\begin{equation}\label{t.9}
\pi(y_t\mid x,y_{<t})=\tau_{x,y_{<t}}\,\pi_{\text{ref}}(y_t\mid x,y_{<t}) \text{almost everywhere on the inactive region.} 
\end{equation}
On the active region where $\pi=(1+\varepsilon)\pi_{\mathrm{prop}}$, KKT allows $\xi\in[0,1]$. Combining with \eqref{t.9} gives the statewise water-filling solution
\begin{equation}\label{t.10}
\begin{split}
   \quad
\pi^\star_\theta(y_t\mid x,y_{<t})&=\min\Big((1+\varepsilon)\,\pi_{\mathrm{prop}}(y_t\mid x,y_{<t}),\ \ \tau_{x,y_{<t}}\,\pi_{\text{ref}}(y_t\mid x,y_{<t})\Big)\ \ \\&\text{almost everywhere},\qquad
\int_{\mathcal Y} \pi^\star_\theta\,\mathrm d\nu=1\ .
\quad 
\end{split}
\end{equation}
Strict concavity of $\ell_{x,y_{<t}}$ in the density ratios $\log\frac{\pi}{\pi_{\text{ref}}}$ (through the KL term) ensures the uniqueness of $\tau_{x,y_{<t}}>0$. Concretely, define
\begin{equation}\label{t.11}
\Phi_{x,y_{<t}}(\tau)=\int_{\mathcal Y} \min\Big((1+\varepsilon)\pi_{\mathrm{prop}}(y_t\mid x,y_{<t}),\ \ \tau\,\pi_{\text{ref}}(y_t\mid x,y_{<t})\Big)\,\mathrm d\nu(y_t).
\end{equation}
The map $\tau\mapsto \Phi_{x,y_{<t}}(\tau)$ is continuous and strictly increasing with $\Phi_{x,y_{<t}}(0)=0$ and $\lim_{\tau\to\infty}\Phi_{x,y_{<t}}(\tau)=(1+\varepsilon)\int \pi_{\mathrm{prop}}\,\mathrm d\nu=1+\varepsilon>1$. By the intermediate value theorem there exists a unique $\tau_{x,y_{<t}}>0$ such that $\Phi_{x,y_{<t}}(\tau_{x,y_{<t}})=1$, which enforces normalization in \eqref{t.10}.

To describe the flow of probability mass, partition the token space at state $(x,y_{<t})$ into the measurable sets
\begin{equation}\label{t.12}
S_{x,y_{<t}}=\Big\{y_t:\ (1+\varepsilon)\,\pi_{\mathrm{prop}}(y_t\mid x,y_{<t})\le \tau_{x,y_{<t}}\,\pi_{\text{ref}}(y_t\mid x,y_{<t})\Big\},\quad
T_{x,y_{<t}}=\mathcal Y\setminus S_{x,y_{<t}}.
\end{equation}
Then almost everywhere $\pi^\star_\theta=(1+\varepsilon)\pi_{\mathrm{prop}}$ on $S_{x,y_{<t}}$ and $\pi^\star_\theta=\tau_{x,y_{<t}}\pi_{\text{ref}}$ on $T_{x,y_{<t}}$, and summing these two pieces yields the continuous mass-balance identity
\begin{equation}\label{t.13}
\int_{S_{x,y_{<t}}}\!\Big((1+\varepsilon)\,\pi_{\mathrm{prop}}-\pi_{\text{ref}}\Big)\,\mathrm d\nu
+\int_{T_{x,y_{<t}}}\!\Big(\tau_{x,y_{<t}}-1\Big)\pi_{\text{ref}}\,\mathrm d\nu
=0.
\end{equation}

Introduce the unobserved token-level utility $u^\star(x,y_{<t},y_t)$, assumed integrable with respect to $\pi_{\text{ref}}(\cdot\mid x,y_{<t})$. The expected performance is
\begin{equation}\label{t.14}
J(\theta)=\mathbb E_{x\sim\mathcal D}\,\mathbb E_{y\sim \pi_\theta(\cdot\mid x)}\Big[\sum_{t=1}^{|y|}u^\star(x,y_{<t},y_t)\Big]
=\mathbb E_{x,y_{<t}}\int_{\mathcal Y}\pi_\theta(y_t\mid x,y_{<t})\,u^\star(x,y_{<t},y_t)\,\mathrm d\nu(y_t).
\end{equation}
To connect the update in \eqref{t.10} to $J$, impose the continuous monotone-likelihood-ratio structure: for each $(x,y_{<t})$ define
\begin{equation}\label{t.15}
h(x,y_{<t},y_t)=\frac{\pi_{\mathrm{prop}}(y_t\mid x,y_{<t})}{\pi_{\text{ref}}(y_t\mid x,y_{<t})},
\end{equation}
and assume that $u^\star(x,y_{<t},\cdot)$ is comonotone with $h(x,y_{<t},\cdot)$ under $\pi_{\text{ref}}(\cdot\mid x,y_{<t})$, that is,
\begin{equation}\label{t.16}
\big(u^\star(x,y_{<t},y_t)-u^\star(x,y_{<t},y_t')\big)\cdot\big(h(x,y_{<t},y_t)-h(x,y_{<t},y_t')\big)\ge 0 \quad \text{for almost every }(y_t,y_t').
\end{equation}
Define the tilt
\begin{equation}\label{t.17}
w_{x,y_{<t}}(y_t)=\frac{\pi^\star_\theta(y_t\mid x,y_{<t})}{\pi_{\text{ref}}(y_t\mid x,y_{<t})}
=\min\Big(\tau_{x,y_{<t}},\ (1+\varepsilon)\,h(x,y_{<t},y_t)\Big).
\end{equation}
The map $z\mapsto \min(\tau_{x,y_{<t}},(1+\varepsilon)z)$ is nondecreasing, and by assumption $u^\star$ is nondecreasing in $h$. Therefore $w_{x,y_{<t}}(\cdot)$ and $u^\star(x,y_{<t},\cdot)$ are comonotone random variables under the base probability measure $\mathrm d\mu_{x,y_{<t}}(y_t)=\pi_{\text{ref}}(y_t\mid x,y_{<t})\,\mathrm d\nu(y_t)$. Chebyshev's association inequality in integral form gives
\begin{equation}\label{t.18}
\int w_{x,y_{<t}}(y_t)\,u^\star(x,y_{<t},y_t)\,\mathrm d\mu_{x,y_{<t}}
\ \ge\ \bigg(\int w_{x,y_{<t}}\,\mathrm d\mu_{x,y_{<t}}\bigg)\bigg(\int u^\star\,\mathrm d\mu_{x,y_{<t}}\bigg).
\end{equation}
Since $\int w_{x,y_{<t}}\,\mathrm d\mu_{x,y_{<t}}=\int \pi^\star_\theta\,\mathrm d\nu=1$, we deduce
\begin{equation}\label{t.19}
\int \pi^\star_\theta(y_t\mid x,y_{<t})\,u^\star(x,y_{<t},y_t)\,\mathrm d\nu
\ \ge\ 
\int \pi_{\text{ref}}(y_t\mid x,y_{<t})\,u^\star(x,y_{<t},y_t)\,\mathrm d\nu.
\end{equation}
Averaging over $(x,y_{<t})$ yields the global monotone improvement
\begin{equation}\label{t.20}
J(\pi^\star_\theta)\ \ge\ J(\pi_{\text{ref}}). 
\end{equation}

The same conclusion can be expressed in the language of mass transfer. The per-state change in performance is
\begin{equation}\label{t.21}
\Delta J_{x,y_{<t}}=\int_{\mathcal Y}\Big(\pi^\star_\theta(y_t\mid x,y_{<t})-\pi_{\text{ref}}(y_t\mid x,y_{<t})\Big)\,u^\star(x,y_{<t},y_t)\,\mathrm d\nu. 
\end{equation}
Using \eqref{t.10}--\eqref{t.13} this decomposes as
\begin{equation}\label{t.22}
\Delta J_{x,y_{<t}}=\int_{S_{x,y_{<t}}}\Big((1+\varepsilon)\pi_{\mathrm{prop}}-\pi_{\text{ref}}\Big)u^\star\,\mathrm d\nu
+\int_{T_{x,y_{<t}}}\Big(\tau_{x,y_{<t}}-1\Big)\pi_{\text{ref}}\,u^\star\,\mathrm d\nu.
\end{equation}
Define the nonnegative transfer magnitude
\begin{equation}\label{t.23}
M_{x,y_{<t}}=\int_{T_{x,y_{<t}}}\Big(\tau_{x,y_{<t}}-1\Big)\pi_{\text{ref}}\,\mathrm d\nu
=-\int_{S_{x,y_{<t}}}\Big((1+\varepsilon)\pi_{\mathrm{prop}}-\pi_{\text{ref}}\Big)\,\mathrm d\nu\ \ge 0,
\end{equation}
and the increment-weighted conditional means
\begin{equation}
    \bar u^{(+)}_{S_{x,y_{<t}}}=\frac{\int_{S_{x,y_{<t}}}\Big(\pi_{\text{ref}}-(1+\varepsilon)\pi_{\mathrm{prop}}\Big)\,u^\star\,\mathrm d\nu}{\int_{S_{x,y_{<t}}}\Big(\pi_{\text{ref}}-(1+\varepsilon)\pi_{\mathrm{prop}}\Big)\,\mathrm d\nu},\qquad
\bar u^{(-)}_{T_{x,y_{<t}}}=\frac{\int_{T_{x,y_{<t}}}\Big(\tau_{x,y_{<t}}-1\Big)\pi_{\text{ref}}\,u^\star\,\mathrm d\nu}{\int_{T_{x,y_{<t}}}\Big(\tau_{x,y_{<t}}-1\Big)\pi_{\text{ref}}\,\mathrm d\nu}.
\end{equation}

Then
\begin{equation}\label{t.24}
\Delta J_{x,y_{<t}}=-\,M_{x,y_{<t}}\big(\bar u^{(+)}_{S_{x,y_{<t}}}-\bar u^{(-)}_{T_{x,y_{<t}}}\big).
\end{equation}
By definition, on $S_{x,y_{<t}}$ one has $(1+\varepsilon)\pi_{\mathrm{prop}}\le \tau_{x,y_{<t}}\pi_{\text{ref}}$, which is equivalent to $h(x,y_{<t},y_t)\le \tau_{x,y_{<t}}/(1+\varepsilon)$; the reverse holds on $T_{x,y_{<t}}$. Under the comonotonicity condition \eqref{t.16} the conditional mean of $u^\star$ on $T_{x,y_{<t}}$ is at least that on $S_{x,y_{<t}}$, hence $\bar u^{(-)}_{T_{x,y_{<t}}}\ge \bar u^{(+)}_{S_{x,y_{<t}}}$, and therefore $\Delta J_{x,y_{<t}}\ge 0$. Averaging over $(x,y_{<t})$ recovers \eqref{t.20}.

It is useful to verify that the tokenwise KL penalty admits an unbiased single-sample estimator in the continuous case. Define $\psi(r)=r-\log r-1$. For any $(x,y_{<t})$,
\begin{equation}\label{t.25}
\scalebox{0.93}{$
\begin{aligned}
\int_{\mathcal Y} \psi\!\Big(\frac{\pi_{\text{ref}}(y_t\mid x,y_{<t})}{\pi_\theta(y_t\mid x,y_{<t})}\Big)\,\pi_\theta(y_t\mid x,y_{<t})\,\mathrm d\nu(y_t)=\int \Big(\tfrac{\pi_{\text{ref}}}{\pi_\theta}-\log\tfrac{\pi_{\text{ref}}}{\pi_\theta}-1\Big)\,\pi_\theta\,\mathrm d\nu
=\mathrm{KL}\big(\pi_\theta(\cdot\mid x,y_{<t})\ \|\ \pi_{\text{ref}}(\cdot\mid x,y_{<t})\big), 
\end{aligned}
$}
\end{equation}
and $\psi(r)\ge 0$ for all $r>0$ because $\psi''(r)=1/r>0$ and $\psi(1)=0$. This validates the KL term in the surrogate \eqref{t.4}. In small-step stochastic ascent, if $\theta^+=\theta+\eta \hat g$ with $\mathbb E[\hat g]=\nabla_\theta$ of the empirical surrogate and bounded second moment, standard smoothness arguments imply an expected increase in the surrogate. Translating to performance via the association argument above yields the first-order estimate
\begin{equation}\label{t.26}
\begin{aligned}
   \mathbb{E}\big[J(\theta^{+})-J(\theta)\big]\ \approx\ \eta\ \mathbb{E}_{x,y_{<t}}\Big[\,\mathrm{Cov}_{Y_t\sim \pi_{\text{ref}}(\cdot\mid x,y_{<t})}\big(\min(\tau_{x,y_{<t}},(1+\varepsilon)h(x,y_{<t},Y_t)),\ u^\star(x,y_{<t},Y_t)\big)\,\Big]\ \ge 0,
\end{aligned}
\end{equation}
because both arguments of the covariance are nondecreasing functions of the common scalar $h(x,y_{<t},Y_t)$.

In conclusion, in continuous output spaces, provided the reference and proposal policies admit densities with respect to a common base measure and satisfy the comonotone likelihood-ratio structure with the latent utility, GRPO with no reward and constant advantage, equipped with ratio clipping and a KL trust region, reduces at each token state $(x,y_{<t})$ to maximizing the overlap integral $\int \min(\pi,(1+\varepsilon)\pi_{\mathrm{prop}})$ minus a KL divergence, a strictly concave problem with the unique water-filling maximizer $\pi^\star_\theta=\min((1+\varepsilon)\pi_{\mathrm{prop}},\,\tau\,\pi_{\text{ref}})$. This update transfers probability mass from low ratio regions to high ratio regions and, by the integral association inequality together with the mass-balance identity, guarantees a monotone increase of the expected latent utility $J$ without requiring any scalar reward signal.
\end{proof}
\end{document}